\documentclass{article}
\usepackage{amsthm}
\newtheorem{theorem}{Theorem}
\usepackage[preprint]{neurips_2026}

\usepackage[utf8]{inputenc} 
\usepackage[T1]{fontenc}    
\usepackage{hyperref}       
\usepackage{url}            
\usepackage{booktabs}       
\usepackage{amsfonts}       
\usepackage{nicefrac}       
\usepackage{microtype}      
\usepackage{xcolor}         

\usepackage{tikz}
\usetikzlibrary{arrows.meta,positioning}
\usepackage{tcolorbox}
\usepackage{enumitem}
\usepackage{microtype}
\usepackage{graphicx}
\usepackage{subcaption}
\usepackage{booktabs} 
\usepackage{multirow}
\usepackage{amsmath}
\usepackage{amssymb}
\usepackage{mathtools}
\usepackage{amsthm}
\usepackage[capitalize,noabbrev]{cleveref}
\usepackage{float}
\usepackage{wrapfig}
\usepackage{amsthm}
\usepackage{subcaption}
\tcbuselibrary{breakable}
\usepackage{colortbl}
\definecolor{deltarow}{gray}{0.93}
\definecolor{wingreen}{RGB}{0,110,60}
\definecolor{lossred}{RGB}{180,30,30}

\newcommand{\dpos}[1]{\textcolor{wingreen}{\textit{#1}}}
\newcommand{\dneg}[1]{\textcolor{lossred}{\textit{#1}}}
\newcommand{\dneu}[1]{\textcolor{gray}{\textit{#1}}}
\newcommand{\deltarow}{\rowcolor{deltarow}}

\newcommand{\hq}[1]{#1}

\title{DIG to Heal: Scaling General-purpose Agent Collaboration via Explainable Dynamic Decision Paths}

\author{%
  Hanqing Yang\thanks{
    Contact: \texttt{hanqing3@andrew.cmu.edu}.
    CMU: Carnegie Mellon University;
    UArizona: University of Arizona.
  } \\
  CMU
  \And
  Hyungwoo Lee \\
  CMU
  \And
  Yuhang Yao \\
  Zoom
  \AND
  Zhiwei Liu \\
  Salesforce
  \And
  Kay Liu \\
  Amazon
  \And
  Jingdi Chen \\
  UArizona
  \AND
  Carlee Joe-Wong\\
  CMU
}
\begin{document}

\maketitle

\begin{abstract}
The increasingly popular agentic AI paradigm promises to harness the power of multiple, general-purpose large language model (LLM) agents to collaboratively complete complex tasks. While many agentic AI systems reduce complexity through predefined workflows or fixed agent roles, the ideal is to support truly autonomous agents capable of emergent collaboration across many interacting agents. Yet in practice, such unstructured interactions often lead to redundant work and cascading failures that are difficult to interpret or correct.
In this work, we study multi-agent systems composed of general-purpose LLM agents that solve problems through emergent collaboration, without relying on predefined roles, control flows, or communication constraints.
We introduce the Dynamic Interaction Graph (DIG), which captures emergent collaboration as a time-evolving causal network of agent activations and interactions. DIG makes emergent collaboration observable and explainable for the first time, enabling real-time identification, explanation, and correction of collaboration-induced error patterns directly from agents' collaboration paths. Thus, DIG fills a critical gap in understanding how general LLM agents solve problems together in truly agentic multi-agent systems. The project webpage can be found at: \url{https://happyeureka.github.io/dig}.
\end{abstract}

\section{Introduction}
\label{sec:intro}

Large Language Models (LLMs) are increasingly used as general decision-making units due to their broad world knowledge, flexible reasoning abilities, and capacity to interact with external tools and memory. This has led to the emergence of \emph{LLM agents}, which extend LLMs with persistent state, tool access, and environment interaction, enabling them to act autonomously in complex tasks. Recently, significant attention has been devoted to \emph{LLM-powered multi-agent systems}~\citep{chen2025five,yang2025llmpowereddecentralizedgenerativeagents}, where multiple agents collaborate on tasks such as research, software engineering, robotics, simulation, and question answering \citep{durante2024agent,he2025llm}.


Despite the potential of allowing LLM agents to act autonomously, in practice unstructured agent collaboration often leads to task failure. Thus, many multi-agent systems utilize \textit{predefined workflows}, with fixed roles, task decompositions, tool usage patterns, and message routing policies, in order to stabilize performance \citep{li2024survey}. However, such designs are often domain-specific and fundamentally constrain how agents interact.
In contrast, as illustrated in Fig.~\ref{fig:coop-prob-solving}, we study LLM-powered multi-agent systems composed of \textbf{general-purpose} LLM agents that operate without predefined roles, control flow, or communication constraints. In such systems, coordination, task decomposition, and problem-solving strategies must arise implicitly from local agent decisions and interactions, resembling how human teams dynamically coordinate and adapt to solve evolving tasks. 
Our goal is to \textit{evaluate} such cooperative intelligence in general-purpose LLM agents by \textit{illuminating} the causes of collaboration failures so as to enable more resilient autonomous multi-agent systems.

\begin{figure}[t]
    \centering
    \begin{minipage}[t]{0.48\textwidth}
        \centering
        \includegraphics[width=1\linewidth]{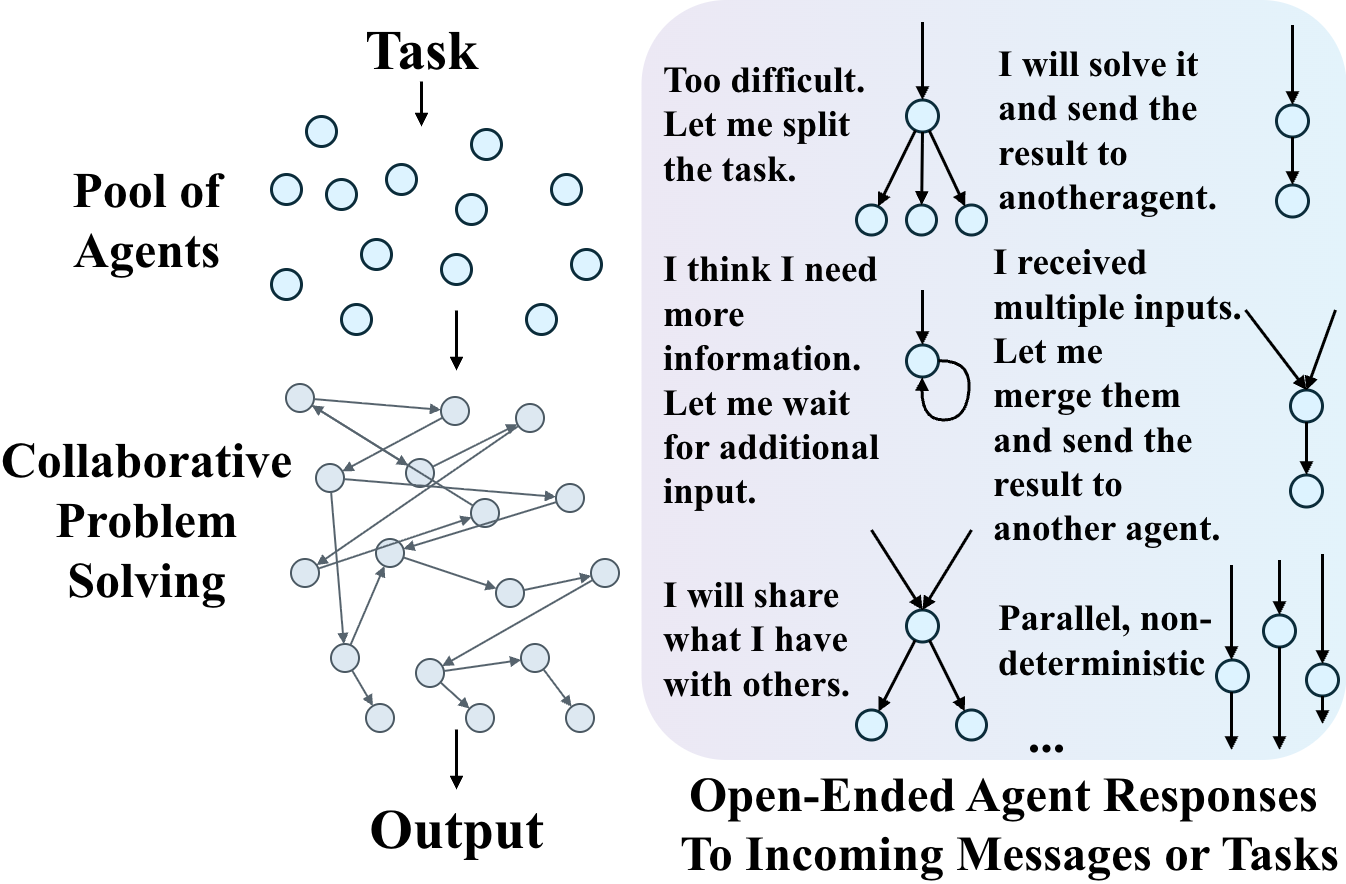}
        \caption{\hq{Cooperative problem solving with general-purpose agents. A pool of autonomous agents works on a shared task without predefined roles, control flow, or communication constraints, each operating independently and non-deterministically and interacting through emergent cooperative strategies. We make emergent cooperation analyzable by modeling agent interaction in a protocol-agnostic manner.}}
        \label{fig:coop-prob-solving}
    \end{minipage}
    \hfill
    \begin{minipage}[t]{0.51\textwidth}
        \centering
        \includegraphics[width=1\linewidth]{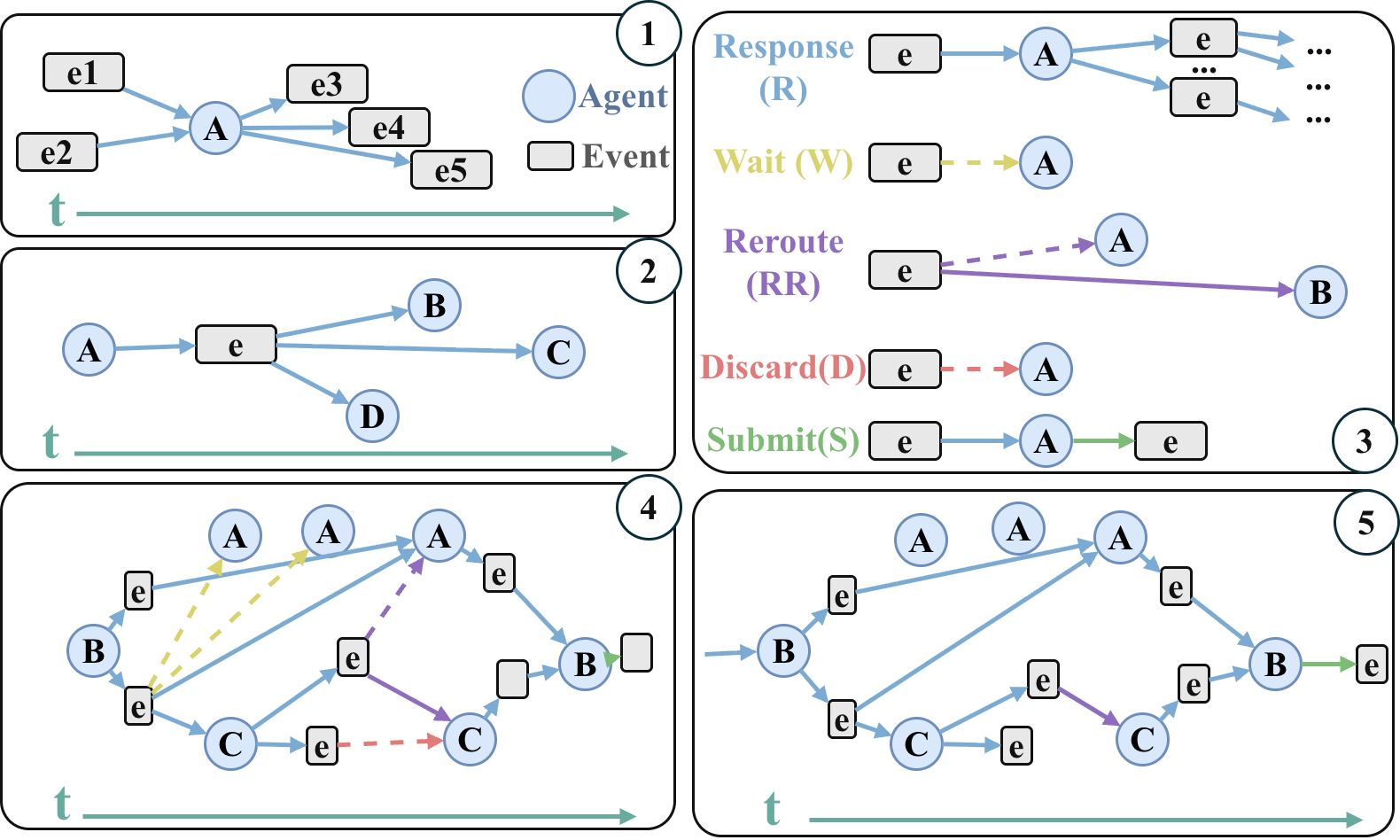}
        \caption{
        \hq{Dynamic Interaction Graph (DIG) and local graph rewrite operators. (1-2) Bipartite graph structure with agent activation nodes (circles) and event nodes (rectangles). 
        (3) Canonical edge rewrite operators.
        (4) Raw DIG induced by an execution trace, including all generation and delivery edges. (5) Cleaned DIG after removing non-productive edges corresponding to non-productive interactions, exposing the underlying interaction structure.}}
        \label{fig:overview}
    \end{minipage}
    \vspace{-0.15in}
\end{figure}

\textbf{The Challenges of Emergent Cooperation.}
Allowing agents to autonomously interact introduces substantial challenges. Each agent receives information, from other agents and the environment, at arbitrary times while acting concurrently with others. Local agent decisions can thus easily push global execution away from intended behavior, producing redundant work, stalled progress, or cascading failures that are difficult to diagnose or recover from. Simply increasing the number of agents does not guarantee improved performance~\citep{kim2025towards}; instead, interaction complexity grows combinatorially, leading to even more interdependent outputs that may cause task failures or redundancies.
Unlike traditional distributed systems, multi-LLM agent systems are stochastic, heterogeneous, and message-driven, making tracing, debugging, and recovery particularly challenging. Existing approaches typically rely on centralized monitors~\citep{gandhi2025whenagentsgoastray,zhang2025agentracer}, guard agents~\citep{yu2025aworldorchestratingtrainingrecipe}, or post hoc failure attribution based on execution logs or curated error datasets~\citep{kong2025aegis,zhang2025agent}. These methods either assume fixed workflows, operate after failures have already occurred, or require large amounts of annotated data, and therefore do not directly address how cooperation emerges or fails during execution.

\textbf{Our Perspective: Analyzing Agent Interactions.}
In this work, we argue that understanding and improving general multi-agent systems requires treating \emph{collaboration itself} as the primary object of analysis. Rather than focusing on individual agents or final outcomes, we model the system in terms of its interaction structure. We introduce the \emph{Dynamic Interaction Graph (DIG)} (Fig.~\ref{fig:overview}), a time-evolving directed graph that formalizes agent activations and interaction events as a causal network. This abstraction makes emergent collaboration observable and explainable, enabling real-time tracing, structural analysis, and reasoning about how cooperative behavior forms, propagates, and fails.

Building on DIG, we establish a \emph{topological characterization} of emergent collaboration, which maps interaction graph structures to distinct classes of collaboration-induced error patterns. This result provides a model-agnostic categorization of failure modes, enabling early failure detection at the level of collaboration structure rather than individual agent outputs. This structural perspective is designed to be compatible with existing failure detection and attribution methods, e.g., as proposed by \cite{cemri2025multi}. Rather than replacing them, it provides an \emph{error detection and guarding layer}, where topology-level signals can be integrated into semantic, behavioral, and learning-based detectors.
The DIG framework further supports reasoning about effective agent utilization, revealing when agents contribute meaningful progress versus introducing redundancy, inefficiency, or instability. Together, these capabilities provide a unified and explainable foundation for understanding and debugging truly agentic multi-LLM systems, where cooperation emerges dynamically from agents. 

Our main \textbf{contributions} are:
(1) \textbf{Interaction Modeling.} We formalize the interaction structure and cooperative problem-solving process of general LLM agents with DIG, a dynamic directed acyclic graph (DAG) representation that captures both sequential and parallel dependencies, as well as temporal ordering, in unconstrained multi-agent systems. 
(2) \textbf{Topological Failure Characterization.} We establish a mapping from the DIG topology to collaboration-induced error types, yielding a high-level, model-agnostic categorization of failure modes that supports \emph{explainable and hierarchical} error detection and can be naturally integrated with existing error detection and attribution methods.
(3) \textbf{General-Agent Execution Traces.} We evaluate DIG on two scalable and verifiable evaluation tasks, allowing us to introduce the first dataset of execution traces that capture the behavior and interaction patterns of general-purpose agents operating without predefined roles or control flow. 
(4) \textbf{Structure-Driven System Healing.} We propose an automatic system healing mechanism driven by observed agent behavior and interaction structure, enabling early detection and correction of collaboration-induced failures during execution. In experiments, we find that this mechanism yields improved efficiency and fewer errors, particularly on difficult tasks with many agents.

After reviewing related work in Section~\ref{sec:related}, we formalize agent interactions (Section~\ref{sec:problem_formulation}) and DIG's construction (Section~\ref{sec:method_dig}). We then explain how DIG enables failure monitoring and healing (Section~\ref{sec:failures}) and describe the design (Section~\ref{sec:experiment_design}) and results (Section~\ref{sec:results}) of our evaluation, concluding in Section~\ref{sec:conclusion}.

\section{Related Work}
\label{sec:related}

\textbf{LLM-powered Multi-agent Systems.}
Recent frameworks such as LangGraph~\citep{langgraph2024}, AutoGen~\citep{wu2023autogen}, and AgentVerse~\citep{chen2023agentverse} enable multi-agent collaboration through structured workflows and orchestration policies, with applications including research assistants, software engineering, robotics, and simulation \citep{chen2025five}. While effective for specific domains, these systems rely on predefined roles, control flow, and communication structures~\citep{yan2025beyond}, providing limited insight into emergent cooperation in unconstrained, general-agent systems.

\textbf{Agent Interaction Modeling.}
Our interaction model is related to classical abstractions such as stochastic computation graphs (SCGs) \citep{schulman2016gradientestimationusingstochastic}, Petri nets \citep{peterson1977petri}, the Actor model \citep{agha1986actors}, description logics \citep{baader2008description,antoniou2009web}, and process algebras \citep{fokkink2013introduction,de2014gentle}. These formalisms model computation, concurrency, or system properties, but are not designed to preserve the asynchronous, non-deterministic, natural-language interaction paths generated by LLM agents without fixed protocols in real time, so as to enable failure interventions. In contrast, DIG provides lightweight interaction model tailored to emergent cooperation in multi-LLM systems. Please see Appendix~\ref{app:formalism_comparison} for further comparisons.

\textbf{Failure Characterization in Multi-agent Systems.}
Recent taxonomies show that multi-agent LLM systems frequently fail due to specification errors, misalignment, and verification breakdowns, with local failures often cascading into system-wide issues \citep{cemri2025multi}. Prior work addresses these failures through post hoc attribution \citep{kong2025aegis,zhang2025agent}, centralized monitors or guard agents \citep{yu2025aworldorchestratingtrainingrecipe,gandhi2025whenagentsgoastray,zhang2025agentracer}, and runtime supervision mechanisms \citep{wang2025megaagent,lin2025stop}, as well as graph-based monitoring and defense \citep{wu2025monitoring}. 
While effective in specific settings, these approaches are typically coupled to fixed agent architectures or orchestration patterns, rely on curated error datasets, and diagnose failures at the level of individual agents or final outcomes. By contrast, our work provides a lightweight, framework-agnostic trace that captures interactions among general agents, enabling systematic and explainable error attribution and scaling naturally to massively parallel multi-agent systems.


\section{Problem Formulation: Cooperative Problem Solving with General Agents}
\label{sec:problem_formulation}

We consider cooperative problem solving in systems of multiple \emph{general-purpose agents} that interact asynchronously via message passing. 
As illustrated in Fig.~\ref{fig:coop-prob-solving}, agent collaboration and task decomposition thus emerge solely from local agent decisions.

\textbf{System Model.}
A cooperative multi-agent system is characterized by the tuple
\(
\mathcal{S} = (\mathcal{A}, \mathcal{E}, \mathcal{P}, \mathbb{Z}),
\)
where $\mathcal{A}=\{a_1,\dots,a_N\}$ is a finite set of agents, $\mathcal{E}$ is the set of interaction events, $\mathcal{P}$ denotes problem instance to be solved, $\mathbb{Z}$ denotes the discrete logical time domain.
$\mathcal{P}$ is encoded as an initial event $e_0\in\mathcal{E}$, delivered at $t=0$ to an initial subset of agents $\mathcal{A}_0\subseteq\mathcal{A}$. 
Execution terminates when a terminal event $e_\infty$ is generated.
Logical time $t \in \mathbb{Z}$ advances whenever an interaction occurs, including agent activation or event delivery.
It induces a causal ordering over interactions but may not correspond to wall-clock time.
Each \textbf{agent} $a\in\mathcal{A}$ (omitting the index $i$ in $a_i$ for notational simplicity) is modeled as an autonomous decision-making unit that operates asynchronously and has no predefined role or control flow. Thus, we abstract away agent internals and model only observable interaction behavior via a non-deterministic transformation on events:
\begin{equation}
f_a : I_a \longrightarrow O_a,
\qquad
I_a,\, O_a \subseteq \mathcal{E}.
\label{eq:agent_map}
\end{equation}


Each agent $a$ maintains a local event \textbf{buffer} $B_a(t)\subseteq\mathcal{E}$. Let $\mathrm{activate}(t,a)\in\{0,1\}$ indicate whether agent $a$ is activated at time $t$. An idle agent is activated when its buffer changes:
$$
\mathrm{activate}(t,a)=1
\quad \text{if} \quad
B_a(t)\neq B_a(t-1)
\ \text{and}\
\mathrm{activate}(t-1,a)=0.
$$
A buffer can change in two cases: the agent receives a new event, or the agent processes an existing event and removes it from the buffer. For example, an event may specify a computational task, i.e., sorting a list, which the agent performs after activation.
Agents interact through \textbf{events} $e$, defined as
\begin{equation}
e=(p_e,\pi_e)\in\mathcal{E},
\qquad
\pi_e(t)\mapsto (R_t(e),\pi_e'),
\quad
R_t(e)\subseteq\mathcal{A},
\label{eq:event_def}
\end{equation}
where $p_e$ is the event payload and $\pi_e$ is its delivery policy. At logical time $t$, $\pi_e$ determines the recipient set $R_t(e)$ and updates itself to $\pi_e'$. The event is then appended to the buffers of agents in $R_t(e)$. If $\pi_e'$ is empty, the event expires. For example, $p_e$ may be the result of a computation completed by agent $a$, and $\pi_e$ may immediately deliver the result $p_e$ to all agents $a'\neq a$.

\textbf{Tracking Agent Dynamics.}
The coupled dynamics of agent transformations~\eqref{eq:agent_map} and policy-driven event delivery~\eqref{eq:event_def} jointly define a stochastic execution process over logical time. We record this process as an execution trace $\mathcal{T}$ and define the system-level diagnosis problem $\Phi$ as:
\begin{equation}
    \mathcal{T}=\{(a_t,e_t,t)\}_{t\in\mathbb Z},
    \qquad
    \Phi:\mathcal{T}\to\mathcal{F}.
    \label{eq:trace_failure_def}
\end{equation}
Here, $a_t$ is the agent activated at time $t$, if any, and $e_t$ denotes the event generated or delivered at that time. The challenge is that asynchronous and non-deterministic execution makes failures hard to localize from the terminal output $e_\infty$ alone. Failures may arise from intermediate interaction structure, e.g., stalled activations, lost subproblems, duplicated work, or premature termination. Thus, $\Phi$ must infer $\mathcal{F}$, the space of failure signals and explanations, from observable interaction traces without access to agent internals or semantic task labels. We address this with the \emph{Dynamic Interaction Graph (DIG)}, a formal representation of agent interactions during cooperative problem solving.

\section{Method: Dynamic Interaction Graph (DIG)}
\label{sec:method_dig}

We present DIG as a system-level representation and evaluation framework for cooperative problem solving in general multi-agent systems. Given an execution trace $\mathcal{T}$, DIG transforms observed agent interactions into a labeled causal interaction graph $\Psi(\mathcal{T})$. This graph enables graph-level inference to detect failures, explain coordination dynamics, and derive interventions without requiring access to agent internals or semantic task labels.

\subsection{DIG Definition}\label{sec:dig_definition}

DIG represents agent cooperation as a dynamic causal graph. At logical time $t$, DIG is defined as:
\begin{equation}
G(t) = (\mathcal{V}_A(t) \cup \mathcal{V}_E(t),\, \mathcal{E}_G(t)),
\label{eq:dig_graph}
\end{equation}
where $\mathcal{V}_A(t)$ is the set of \textbf{activation nodes}, $\mathcal{V}_E(t)$ is the set of \textbf{event nodes}, and $\mathcal{E}_G(t)$ is the set of directed edges, \hq{as illustrated in Fig.~\ref{fig:overview} (1--2)}. The graph is bipartite: 
edges only connect activation nodes and event nodes.
Each \textbf{activation node} represents one activation of agent $a_v \in \mathcal{A}$ over the interval $\tau_v=[t_v^{\mathrm{start}},t_v^{\mathrm{end}}]$:
$
v=(\mathrm{id}_v,a_v,\tau_v).
$
Each \textbf{event node} represents one interaction event together with its realized delivery times:
$
e_G=(e,\tau_e), \quad
\tau_e=\bigl(t_e^{\mathrm{gen}},\{t_e^{\mathrm{recv}}(b)\}_{b\in R_\infty(e)}\bigr).
$


\textbf{Edges.}
DIG has two edge types:
(1) \textit{Generation edges} $(v \rightarrow e_G) \in \mathcal{E}_G$ when activation $v$ generates event $e$, and
(2) \textit{Delivery edges} $(e_G \rightarrow v_b) \in \mathcal{E}_G$ when event $e$ triggers activation $v_b$.
Collectively, $\mathcal{E}_G$ defines a complete \emph{temporal causal trace} of execution, on which all causal influence propagates.



\textbf{Illustration.}
Fig.~\ref{fig:overview}(4--5) illustrate DIG. Circles denote activation nodes $v \in \mathcal{V}_A$, spanning $[t_v^{\mathrm{start}}, t_v^{\mathrm{end}}]$, and squares denote event nodes $e_G \in \mathcal{V}_E$ at generation time $t_e^{\mathrm{gen}}$. Edges from circles to squares represent generation $(v \rightarrow e_G)$, while edges from squares to circles represent delivery $(e_G \rightarrow v_b)$. Fig.~\ref{fig:overview}(4) shows the full real-time DIG, including transient edges (defined in Sec.~\ref{sec:interaction_primitives}), while Fig.~\ref{fig:overview}(5) omits them. Together, they show how an execution trace $\mathcal{T}$ maps to a time-indexed bipartite causal graph $G(t)$, where cooperation structure is encoded by nodes and directed edges. DIG makes cooperation dynamics and causal dependencies observable.

\subsection{Interaction Primitives as Graph Rewrite Operators}
\label{sec:interaction_primitives}

The DIG evolves through \emph{local graph rewrite operators} induced by agent activations. When an agent $a$ is activated at logical time $t$, the activation is represented by an activation node $v \in \mathcal{V}_A(t)$.
Let $B_a(t)$ denote the local event buffer of agent $a$ at time $t$, and let $I_a(t) \;\subseteq\; B_a(t)$
denote the subset of events selected by the agent during the activation. These events correspond exactly to the set of incoming delivery edges of the activation node:
\begin{equation}
E_G^{\mathrm{in}}(v)
\;=\;
\{(e,v) \in \mathcal{E}_G(t)\}.
\label{eq:incoming_edges}
\end{equation}
We further assign a causal fate to each incoming interaction edge by defining a system-level \emph{edge rewrite semantics} through a mapping
$\phi_v : E_G^{\mathrm{in}}(v) \;\longrightarrow\; \{\textsc{consume}, \textsc{delay}, \textsc{reroute}, \textsc{discard}\}$.
This induces a local graph rewrite operator $\mathcal{R}_v : G(t) \;\longrightarrow\; G(t^+)$
that transforms the DIG by modifying, removing, or creating nodes and edges. We elaborate these operators next.

\textbf{Canonical Rewrite Operators.}
Let $O_a(t)\subseteq\mathcal{E}$ denote the new events generated by agent $a$ during activation; these become newly created event nodes in the DIG. The standard interaction primitives induce the following canonical edge rewrite patterns:

\textbf{(1) Respond (R):}
The agent acts on all buffered events,
\( I_a(t)=B_a(t) \) and \( B_a(t^+)=\emptyset \),
and labels all incoming edges as consumed,
\( \phi_v(e,v)=\textsc{consume} \).
The DIG is rewritten by removing all input events from the agent buffer and generating new event nodes, $G(t^+) = G(t)\setminus E_G^{\mathrm{in}}(v)\;\cup\;\{(v,e') \mid e'\in O_a(t)\}.$
\textbf{(2) Wait (W):}
The agent preserves its buffer,
\( B_a(t^+)=B_a(t) \),
and labels all incoming edges as delayed,
\( \phi_v(e,v)=\textsc{delay} \).
No graph modification is applied: $G(t^+)=G(t)$.
\textbf{(3) Reroute (RR):}
The agent forwards selected events,
\( B_a(t^+)=B_a(t)\setminus I_a(t) \),
and labels their edges as rerouted,
\( \phi_v(e,v)=\textsc{reroute} \).
Each such edge is redirected:$(e,v)\;\mapsto\;(e,v').$
\textbf{(4) Discard (D):}
The agent drops selected events,
\( B_a(t^+)=B_a(t)\setminus I_a(t) \),
and labels the edges as discarded,
\( \phi_v(e,v)=\textsc{discard} \).
The edges are removed:$(e,v)\;\mapsto\;\emptyset.$
\textbf{(5) Submit (S):}
A terminal Respond with
\( O_a(t)=\{e_\infty\} \),
yielding $G(t^+)=G(t)\setminus E_G^{\mathrm{in}}(v)\;\cup\;\{(v,e_\infty)\}.$

The overall system execution is the composition of local graph rewrites induced by activation nodes:
\begin{equation}
G_0 
\xrightarrow{\mathcal{R}_{v_1}}
G_1
\xrightarrow{\mathcal{R}_{v_2}}
\cdots
\xrightarrow{\mathcal{R}_{v_T}}
G_T,
\label{eq:global_rewrite}
\end{equation}
where each $\mathcal{R}{v_t} \in {(R), (W), (RR), (D), (S)}$ is determined by the edge labeling function $\phi{v_t}$.

\textbf{Edge Attribution.} We attribute each edge one of two attribution types: \emph{productive} or \emph{non-productive} (transient). Productive edges indicate interactions that advance the cooperative solution, while non-productive edges indicate interactions that do not advance it at the current step. Edges corresponding \textbf{Wait (W)} and \textbf{Discard (D)} are non-productive and drawn as dashed. For \textbf{Reroute (RR)}, the edge to the rerouting agent is non-productive and dashed, while the redirected edge to the final recipient is productive and solid. Fig.~\ref{fig:overview} (3--4) illustrate this distinction: solid lines denote productive edges, and dashed lines denote non-productive edges.



\textbf{Illustration.}
Cooperative multi-agent execution is thus a \emph{graph rewriting dynamical system}. As illustrated in Fig.~\ref{fig:overview} (3), agents coordinate not through predefined roles or control flow, but by applying local edge-level rewrites to the DIG. Emergent cooperation structure, system-level behavior, and failure modes arise from the composition of these local edge-level rewrite operators, making DIG the effective state space of cooperative intelligence.


\subsection{DIG as a Representation of Cooperative Intelligence}
\label{sec:dig_state_space}
In this section, we show that DIG representation $\{G(t)\}_{t\in\mathbb{Z}}$ is not merely a visualization artifact, but an \emph{intermediate operational state space} for system-level reasoning over cooperation. Thus, DIG provides the core representation of cooperative intelligence.


\textbf{DIG as a Labeled Causal Interaction Graph.}
For any logical time $t$, let
\begin{equation}
G(t) = \bigl(\mathcal{V}_A(t)\cup\mathcal{V}_E(t),\,\mathcal{E}_G(t),\,\phi_t\bigr),
\label{eq:labeled_dig}
\end{equation}
where $(\mathcal{V}_A(t)\cup\mathcal{V}_E(t),\mathcal{E}_G(t))$ is the directed bipartite causal topology defined in Sec.~\ref{sec:dig_definition}, and $\phi_t: \mathcal{E}_G(t) \longrightarrow \{\textsc{consume},\textsc{delay},\textsc{reroute},\textsc{discard}\}$
is the edge-action labeling induced by local rewrite operators (Sec.~\ref{sec:interaction_primitives}).  
Thus, each $G(t)$ is a fully labeled causal interaction graph, encoding both \emph{who interacted with whom} and \emph{how each interaction edge was operationally treated}.

This formalism characterizes cooperative behavior structurally. Given an execution trace $\mathcal{T}$, we group together all traces that induce the same time-indexed labeled DIG:
$\mathcal C(\mathcal T)=\{\mathcal T' \mid \Psi(\mathcal T')\cong\Psi(\mathcal T)\}$.
Here, $\cong$ denotes graph isomorphism that preserves bipartite node types, directed edges, and edge labels. Thus, cooperative behavior is characterized by observable interaction structure rather than agent internals. Equivalently, cooperation is encoded by the temporal evolution of DIG under the edge-level rewrite operators $\{\phi_t\}$, which specify how interactions are consumed, delayed, rerouted, or discarded.


\textbf{DIG as an Explainability Trace.}
The sequence $\{G(t)\}_{t\in\mathbb Z}$ forms an explainability trace of cooperative intelligence. Any system-level behavior or failure signal $f\in\mathcal F$ inferred from an execution trace $\mathcal T$ can be expressed as
$f=\mathcal I(\{G(t)\}_{t\in\mathbb Z})=\mathcal I(\{\phi_t(G(t))\}_{t\in\mathbb Z})$
for some graph-level inference operator $\mathcal I$. Thus, explanations depend only on the labeled causal structure $(\mathcal V_A\cup\mathcal V_E,\mathcal E_G,\phi_t)$ and not on unobservable agent internals. In this sense, DIG is not merely a representation of execution, but a causal explanation object for emergent multi-agent cooperation.


\begin{theorem}[Topological Inference Reduction]
\label{thm:topological_reduction}
Fix a class $\mathcal{C}$ of cooperative multi-agent systems whose observable semantics are defined by asynchronous activations and event passing, as in Sec.~\ref{sec:problem_formulation}. Let $\mathbb{T}$ be the space of observable execution traces, with $\mathcal{T}\in\mathbb{T}$ denoting one trace, and let $\Psi:\mathbb{T}\to\{G(t)\}_{t\in\mathbb{Z}}$ be the DIG construction operator.

For any system-level functional $\Phi:\mathbb{T}\to\mathcal{F}$ that depends only on observable interactions, equivalently, is invariant over traces with isomorphic DIGs, there exists a graph functional $\mathcal{I}:\{G(t)\}_{t\in\mathbb{Z}}\to\mathcal{F}$ such that, for every trace $\mathcal{T}\in\mathbb{T}$,
$\Phi(\mathcal{T})=\mathcal{I}(\Psi(\mathcal{T}))$.
\end{theorem}



The proof of Theorem~\ref{thm:topological_reduction} is given in Appendix~\ref{app:proof}. This result shows that \emph{all structurally observable system-level reasoning about cooperative intelligence reduces to inference over the DIG}. Thus, diagnosis, explanation, and intervention are performed not on hidden agent internals, but on the topology and edge-action structure of the evolving causal interaction graph.


\subsection{Problem-Solving in DIG}
\label{sec:dig_problem_solving}

We interpret cooperative problem solving as a \emph{structural transformation process} over the DIG. Given an execution trace $\mathcal{T}$ and its induced DIG sequence $\{G(t)\}_{t\in\mathbb{Z}}=\Psi(\mathcal{T})$, task progress is realized not through access to agent internals, but through the composition of local graph rewrite operators (Eq.~\eqref{eq:global_rewrite}). Specifically, each activation node $v_t\in \mathcal{V}_A$ induces a rewrite $\mathcal{R}_{v_t}$ through its edge-action operator $\phi_{v_t}$ and output event set $O_{a_t}(t)$.


Each activation node $v\in\mathcal{V}_A$ consumes an input set of events and produces an output set of events, as recorded by the incident edges in the DIG. Let
\begin{equation}
\begin{aligned}
I_v &= \{\, e \in \mathcal{V}_E \mid (e,v)\in \mathcal{E}_G \,\}, \qquad
O_v &= \{\, e' \in \mathcal{V}_E \mid (v,e')\in \mathcal{E}_G \,\}.
\end{aligned}
\label{eq:Iv_Ov_def}
\end{equation}
denote the input events delivered to $v$ and the output events generated by $v$, respectively. We interpret $v:\; I_v \;\mapsto\; O_v$ as a local problem-solving transformation.
%
An activation is \textbf{problem-generating} if it expands the unresolved work, $|O_v| > |I_v|$,
in which case each $p\in O_v$ is treated as a \emph{subproblem} spawned from context $I_v$. It is \textbf{problem-reducing} if $|O_v| \le |I_v|$,
in which case each $s\in O_v$ is treated as a \emph{subsolution} produced from context $I_v$. 

\section{Failure, Detection, and Healing}\label{sec:failures}

\begin{wrapfigure}{r}{0.4\textwidth}
    \centering
    \includegraphics[width=1\linewidth]{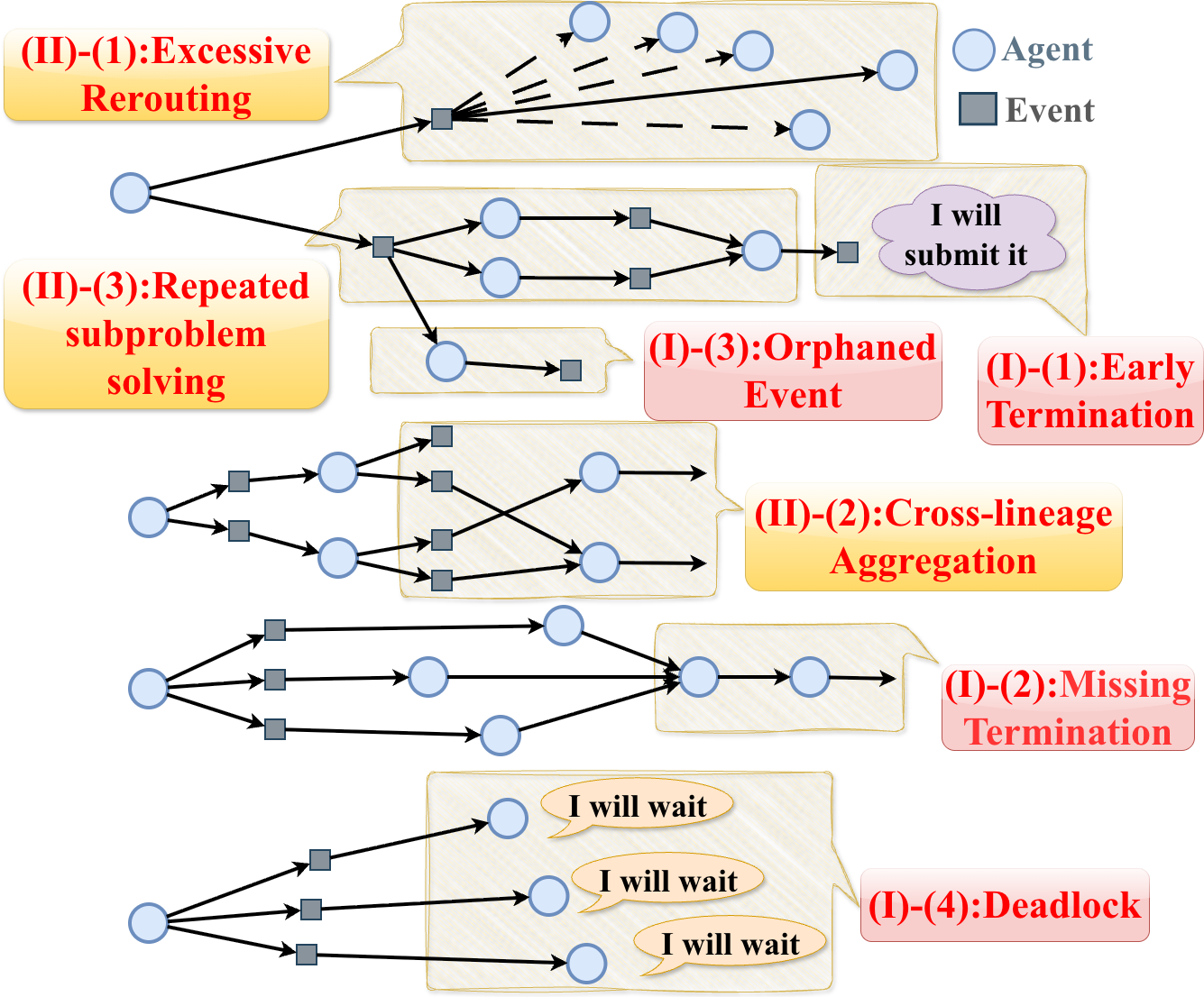}
    \caption{Structural failure patterns in DIG, showing collaboration-invariant violations as interaction structure.}
    \label{fig:failure_patterns}
\end{wrapfigure}
The system defines errors not by agents' \textit{Internal Reasoning}, but by observable \textit{Interaction Primitives} (actions taken by agents) and the corresponding DIG structures.

\textbf{Failure Taxonomy.} 
Empirically, cooperative general agent systems exhibit failures that  arise not from isolated reasoning mistakes within individual agents, but from breakdowns in interaction structure and execution dynamics. Fig.~\ref{fig:failure_patterns} summarizes the major classes of failures observed in our experiments.

The taxonomy is defined in terms of observable execution structure, without access to agent internals or task semantics. We consider two classes of violated structural or temporal invariants: failures in preserving and exhausting work (Reachability and Coverage) and warnings indicating inefficient or risky handling of emitted work (Progress). 
In both classes, the detection method remains task- and domain-agnostic, as it operates on the interaction structure.

\textbf{Failure Detection and Healing.} 
Whenever a new event $e$ is generated at time $t$, the system temporarily blocks its delivery and evaluates the current $G(t)$ for error patterns shown in Fig.~\ref{fig:failure_patterns}. If a failure is detected, the system may \textit{heal} it by (i) injecting new information into it before delivery, (ii) injecting and rerouting it to a recipient set, or (iii) creating a new event and delivering it to a recipient set. We broadly categorize failures into two types: (I) reachability and (II) progress failures. We detail the detection and healing procedure for each failure pattern below.

\underline{\textbf{(I) Reachability and Termination.}} All reachable work in the DIG should persist until consumed. The system should emit a single \textsc{Submit} event once all reachable work has been consumed.

\textit{(1) Early termination (ET).} 
\textit{Detection:} A \textsc{Submit} event is generated despite unresolved reachable work.
Formally, \(e_\infty \in V_E(t)\) and \(\exists e\in R(t)\) without a directed path from \(e\) to \(e_\infty\). \textit{Healing:} apply (ii) by injecting information about unresolved events and rerouting it to the agent that issued \textsc{Submit}.

\textit{(2) Missing termination (MC).} 
\textit{Detection:} All reachable work is exhausted, but no \textsc{Submit} event is generated. 
Formally, \(R(t)=\emptyset\) and \(e_\infty\notin V_E(t)\) for $t$ exceeding a threshold. \textit{Healing:} apply (i) by injecting information that all reachable work is exhausted.

\textit{(3) Orphaned event (OE).} 
\textit{Detection:} An event is generated but has no recipients for longer than a reasonable time window.
Formally, \(\exists e\in V_E(t)\), \(e\neq e_\infty\), such that its recipient set is empty,
\(R_t(e)=\emptyset\) for longer than a reasonable time window. \textit{Healing:} apply (ii) by injecting status information and rerouting the event back to its generating agent.

\textit{(4) Deadlock (DL).} 
\textit{Detection:} Reachable work remains, but no activation occurs within a reasonable time window.
Formally, \(R(t)\neq\emptyset\) while \(V_A(t)=\emptyset\) for all $t$ within a reasonable time window. \textit{Healing:} apply (iii) by creating a new event and broadcasting it to all agents to restart activity.

\underline{\textbf{(II) Progress.}} Generated events should be consumed downstream within a reasonable time; repeated deferral, rerouting, or redundant handling indicates risk.

\textit{(1) Excessive rerouting (ER).} 
\textit{Detection:} An event is repeatedly \textsc{Reroute}d across activations without ever being \textsc{Consume}d.
Formally, \(\exists e\in V_E(t)\) such that its delivery edges are rerouted more than a reasonable number of times. \textit{Healing:} apply (i) by injecting information indicating repeated rerouting.

\textit{(2) Cross-lineage aggregation (CLA).} 
\textit{Detection:} Events from different problem-generating activations are delivered to the same recipient agent.
Formally, there exists \(v\in V_A(t)\) and \(e\neq e'\in I_v\) such that there is no activation \(v_g\)
whose generated events have a directed path to both \(e\) and \(e'\). \textit{Healing:} apply (i) by injecting ancestry information into each event.

\textit{(3) Repeated subproblem solving (RSP).} 
\textit{Detection:} Multiple problem-reducing activations \textsc{Consume} the same upstream event $p$.
Formally, \(\exists p\in V_E(t)\) and problem-reducing activations \(v\neq v'\) such that \(p\in I_v \cap I_{v'}\). \textit{Healing:} apply (i) by injecting information on which part of the result may be repeated.


Note that these depend entirely on agent interaction patterns as encoded in the DIG, allowing us to improve problem-solving without modifying agent internals or constraining their interactions.


\section{Experiment Design}
\label{sec:experiment_design}


\begin{wrapfigure}{r}{0.55\textwidth}
    \centering
    \vspace{-25pt}
    \includegraphics[width=1\linewidth]{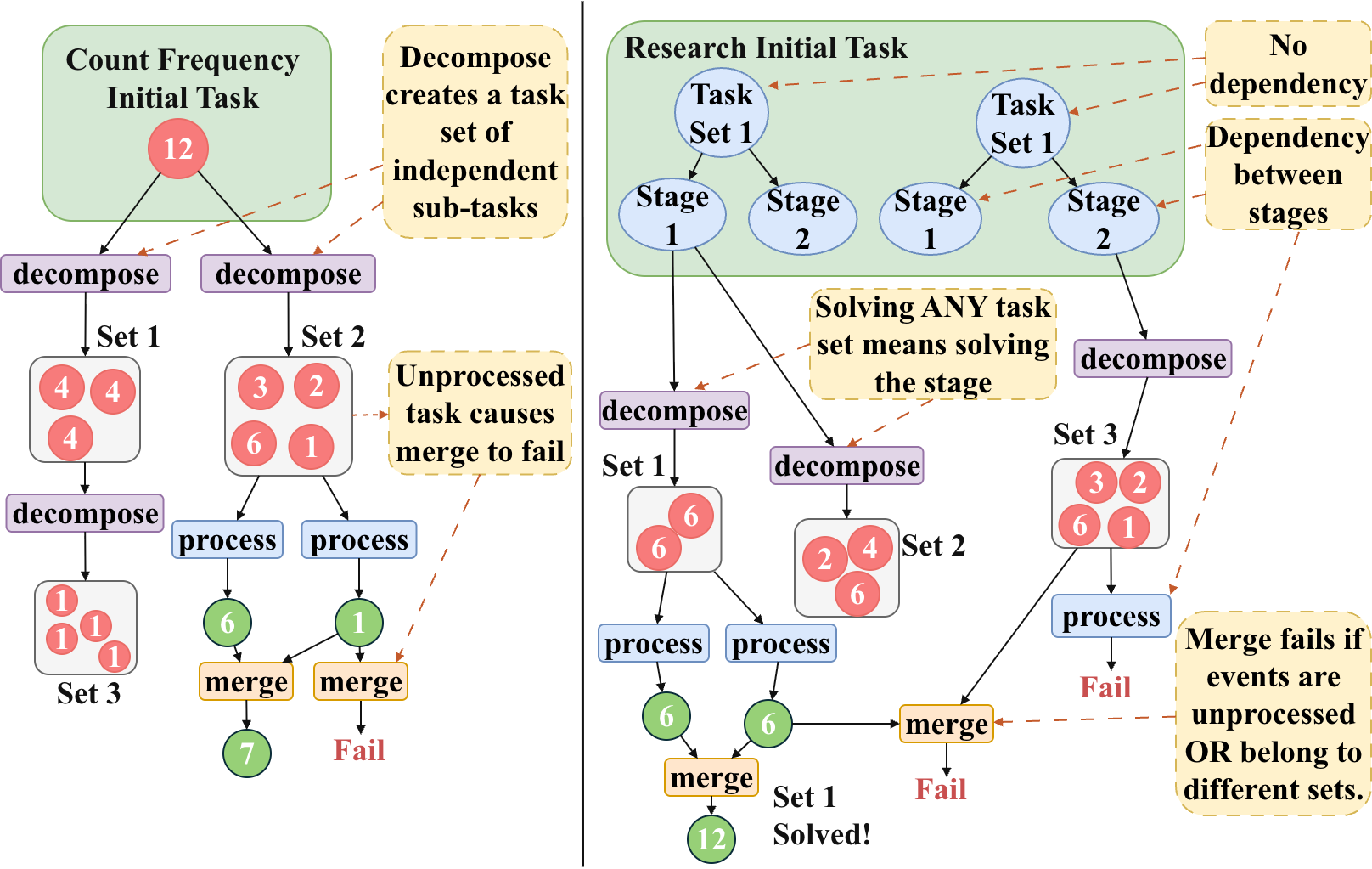}
    \caption{\textbf{Task structures.} \textbf{Left:} Count Frequency (fully parallelizable), where tasks are decomposed into independent subtasks that can be processed and merged once completed. \textbf{Right:} Research Job (dependency-constrained), where tasks are organized into stages with inter-stage dependencies, and violations cause processing or merging failures.}
    \label{fig:task_designs}
    \vspace{-15pt}
\end{wrapfigure}

\textbf{Open-ended, verifiable task-solving.}
We introduce a task-solving setting to evaluate how agent interactions modify task structure and how solving progress can be verified from those structural changes. Each agent \(i\in I\) has an integer capability \(z_i\), and each task component \(g\) has an integer load \(\ell_g\), a state \(s_g\in\{\textsc{Unprocessed},\textsc{Processed}\}\), and optional dependency constraints. There is no fixed structure for solving the task; it can be decomposed in different ways, assigned to different agents, or solved in different orders. When agents choose \textsc{Respond} during an activation, they may invoke tools. We provide three tools for changing task structures and states: \textsc{Decompose}, \textsc{Process}, and \textsc{Merge}.

\textsc{Decompose} splits a task into assigned subtasks, \textsc{Process} marks a feasible task component as processed, and \textsc{Merge} combines processed and compatible components. Each decomposition, processing step, and merge then becomes a trackable task-structure change, making problem-solving structurally verifiable. Agent prompts and detailed tool semantics are provided in Appendices~\ref{app:agent} and~\ref{app:tool_semantics}. We penalize both incomplete and redundant task solving by measuring final task progress as:
\begin{equation}
    \mathrm{coverage}
    =
    \frac{\min(S,N)}{N},
    \quad
    \mathrm{over}
    =
    \frac{\max(S-N,0)}{N},
    \quad
    \mathrm{error}
    =
    (1-\mathrm{coverage})+\mathrm{over},
    \label{eq:progress_error}
\end{equation}
where \(N\) is the target and \(S\) the submitted problem size. 

\begin{wraptable}{r}{0.57\textwidth}
\centering
\scriptsize
\vspace{-16pt}
\setlength{\tabcolsep}{0.5pt}
\begin{tabular}{l c c | c c c}
\toprule
Setting & Agents & Load & \# Failure ($\downarrow$) & Error ($\downarrow$) & Time ($\downarrow$) \\
\midrule
\multicolumn{6}{l}{\textbf{Homogeneous (1000)}} \\
MAS-Only & 4 & 4000  & $\mathbf{2.11 \pm 0.31}$  & $0.083 \pm 0.118$          & $13.0 \pm 9.1$ \\
MAS+DIG  & 4 & 4000  & $2.60 \pm 1.80$           & $\mathbf{0.025 \pm 0.075}$ & $\mathbf{10.7 \pm 7.8}$ \\
\deltarow
$\Delta$\% & & & \dneg{$+23\%$} & \dpos{$-70\%$} & \dpos{$-18\%$} \\
MAS-Only & 4 & 8000  & $\mathbf{2.60 \pm 0.66}$  & $0.688 \pm 0.084$          & $\mathbf{14.3 \pm 8.4}$ \\
MAS+DIG  & 4 & 8000  & $9.10 \pm 1.45$           & $\mathbf{0.588 \pm 0.138}$ & $20.1 \pm 9.4$ \\
\deltarow
$\Delta$\% & & & \dneg{$+250\%$} & \dpos{$-15\%$} & \dneg{$+41\%$} \\
\cline{1-6}
MAS-Only & 6 & 6000  & $\mathbf{2.50 \pm 1.50}$  & $0.100 \pm 0.249$          & $\mathbf{9.2 \pm 1.0}$ \\
MAS+DIG  & 6 & 6000  & $2.60 \pm 1.80$           & $\mathbf{0.017 \pm 0.050}$ & $10.4 \pm 2.8$ \\
\deltarow
$\Delta$\% & & & \dneg{$+4\%$} & \dpos{$-83\%$} & \dneg{$+13\%$} \\
MAS-Only & 6 & 12000 & $\mathbf{3.50 \pm 1.03}$  & $0.850 \pm 0.033$          & $\mathbf{19.4 \pm 12.7}$ \\
MAS+DIG  & 6 & 12000 & $8.80 \pm 0.98$           & $\mathbf{0.775 \pm 0.065}$ & $19.7 \pm 9.1$ \\
\deltarow
$\Delta$\% & & & \dneg{$+151\%$} & \dpos{$-9\%$} & \dneu{$+2\%$} \\
\midrule
\multicolumn{6}{l}{\textbf{Heterogeneous ([..., 1000,100,10,1])}} \\
MAS-Only & 4 & 1111  & $\mathbf{1.80 \pm 0.40}$  & $\mathbf{0.000 \pm 0.000}$ & $\mathbf{8.9 \pm 0.9}$ \\
MAS+DIG  & 4 & 1111  & $2.50 \pm 1.86$           & $0.001 \pm 0.003$          & $9.4 \pm 2.5$ \\
\deltarow
$\Delta$\% & & & \dneg{$+39\%$} & \dneu{$\approx 0$} & \dneg{$+6\%$} \\
MAS-Only & 4 & 2222  & $\mathbf{2.60 \pm 0.80}$  & $0.735 \pm 0.266$          & $\mathbf{33.0 \pm 7.9}$ \\
MAS+DIG  & 4 & 2222  & $8.60 \pm 1.63$           & $\mathbf{0.533 \pm 0.357}$ & $42.9 \pm 8.6$ \\
\deltarow
$\Delta$\% & & & \dneg{$+231\%$} & \dpos{$-27\%$} & \dneg{$+30\%$} \\
\cline{1-6}
MAS-Only & 6 & 111111 & $\mathbf{2.10 \pm 0.30}$ & $0.100 \pm 0.268$          & $\mathbf{12.2 \pm 8.1}$ \\
MAS+DIG  & 6 & 111111 & $5.30 \pm 3.49$           & $\mathbf{0.092 \pm 0.209}$ & $16.8 \pm 9.4$ \\
\deltarow
$\Delta$\% & & & \dneg{$+152\%$} & \dpos{$-8\%$} & \dneg{$+38\%$} \\
MAS-Only & 6 & 222222 & $\mathbf{2.70 \pm 1.10}$  & $0.847 \pm 0.279$          & $\mathbf{32.5 \pm 13.2}$ \\
MAS+DIG  & 6 & 222222 & $10.10 \pm 1.64$          & $\mathbf{0.568 \pm 0.372}$ & $43.8 \pm 11.0$ \\
\deltarow
$\Delta$\% & & & \dneg{$+274\%$} & \dpos{$-33\%$} & \dneg{$+35\%$} \\
\bottomrule
\end{tabular}
\caption{\textbf{CF} with parallelizable tasks. Mean \(\pm\) std (10 runs). \textbf{Bold} = better. \(\Delta\%\) = MAS+DIG vs. MAS-Only.}
\label{tab:cf_experiments}
\vspace{-10pt}
\end{wraptable}
\textbf{Experiment Settings.}
We instantiate this setting with two task structures, shown in Fig.~\ref{fig:task_designs}: \textbf{(i)} a \textit{fully parallelizable structure}, implemented by \textbf{Count Frequency (CF)}, where independent components can be processed in any order and merged after completion; \textbf{(ii)} a \textit{dependency-constrained structure}, implemented by \textbf{Research Job (RJ)}, where components can be processed only after earlier stages are solved. We evaluate both \textbf{homogeneous} and \textbf{heterogeneous} agent teams: \textbf{(i)} in the \textit{homogeneous} setting, all agents have the same capability, \(z_i=1000\); \textbf{(ii)} in the \textit{heterogeneous} setting, capabilities vary, e.g., \(z_i\in\{1,10,100,1000\}\). We compare \textbf{MAS-Only}, which executes without DIG-based healing, with \textbf{MAS+DIG}, which detects and heals structural failures during execution. We further include \textbf{MAS+LLM Judge}, following prior work by~\cite{cemri2025multi}, which invokes an LLM to diagnose failures using a predefined error taxonomy. Since this baseline is not designed for real-time monitoring, we adapt it by running the judge periodically every \(n\) activations using our error taxonomy (see Appendix~\ref{app:llm-judge-prompt}). We report final task error, detected errors, and running time, with statistical analysis in Appendix~\ref{app:significance-tests}.

\begin{figure*}[t]
    \centering
    \begin{subfigure}{0.49\textwidth}
        \centering
        \includegraphics[width=\linewidth]{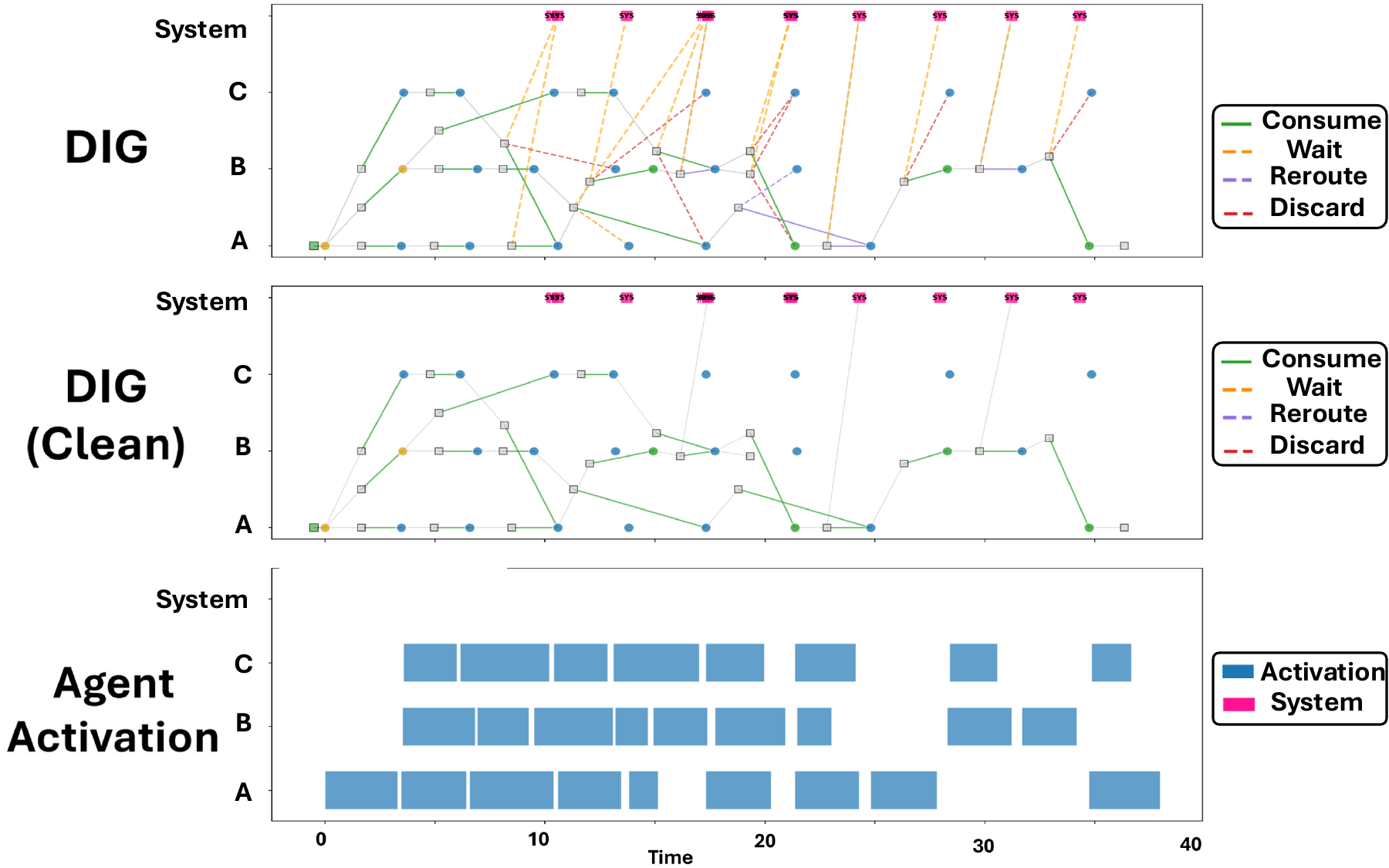}
        \caption{MAS + DIG (3 agents)}
        \label{fig:n3}
    \end{subfigure}
    \hfill
    \begin{subfigure}{0.49\textwidth}
        \centering
        \includegraphics[width=\linewidth]{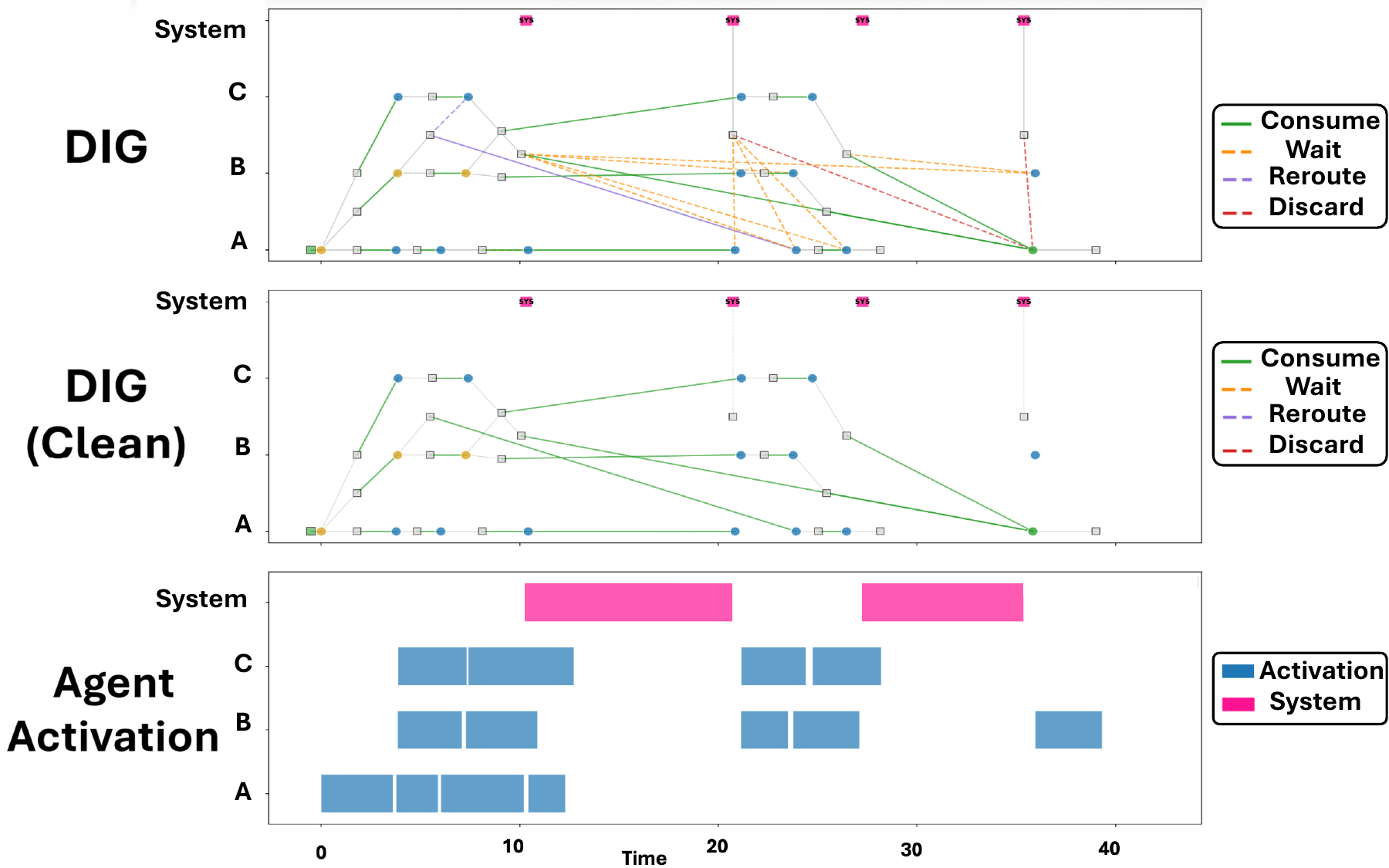}
        \caption{MAS + LLM Judge}
        \label{fig:llm}
    \end{subfigure}
    \caption{With DIG (a), the timeline of agent activations, event propagation, and edge-level rewrites indicates stable coordination and successful task termination. Using an LLM as judge (b), the DIG trace reveals unstable coordination, excessive waiting and rerouting, and delayed interventions.}
    \label{fig:dig_illustrations}
    \hfill
\vspace{-0.25in}
\end{figure*}

\section{Experiment Results}\label{sec:results}
\begin{wraptable}{r}{0.45\textwidth}
\centering
\scriptsize
\vspace{-40pt}
\setlength{\tabcolsep}{0.5pt}
\begin{tabular}{l | c c c}
\toprule
Setting & \# Failure ($\downarrow$) & Error ($\downarrow$) & Time ($\downarrow$) \\
\midrule
\multicolumn{4}{l}{\textbf{Homogeneous (1000)}} \\
MAS-Only & $\mathbf{6.20 \pm 2.96}$ & $0.50 \pm 0.30$          & $24.8 \pm 31.8$ \\
MAS+DIG  & $8.40 \pm 1.69$          & $\mathbf{0.28 \pm 0.28}$ & $\mathbf{17.9 \pm 2.2}$ \\
\deltarow
$\Delta$\% & \dneg{$+35\%$} & \dpos{$-44\%$} & \dpos{$-28\%$} \\
\midrule
\multicolumn{4}{l}{\textbf{Heterogeneous ([..., 1000,100,10,1])}} \\
MAS-Only & $\mathbf{3.50 \pm 1.43}$ & $0.55 \pm 0.15$          & $\mathbf{43.4 \pm 32.9}$ \\
MAS+DIG  & $8.70 \pm 1.85$          & $\mathbf{0.45 \pm 0.15}$ & $68.6 \pm 26.7$ \\
\deltarow
$\Delta$\% & \dneg{$+149\%$} & \dpos{$-18\%$} & \dneg{$+58\%$} \\
\bottomrule
\end{tabular}
\caption{\textbf{RJ} with two problem sets, each containing two sequential stages. Stage load is 4000 in the homogeneous setting and 1111 in the heterogeneous setting; each experiment uses 4 agents. Mean \(\pm\) std (10 runs). \textbf{Bold} = better. \(\Delta\%\) = MAS+DIG vs. MAS-Only.}
\label{tab:rd_experiments}
\vspace{-20pt}
\end{wraptable}
We first analyze DIG results on both Count Frequency and Research Job, and then present a 20-agent case study to demonstrate DIG's scalability.


Tab.~\ref{tab:cf_experiments} and~\ref{tab:rd_experiments} summarize the main results. Across both task structures, MAS+DIG reduces final error at the cost of detecting more failures and, in some settings, longer runtime. On \textbf{Count Frequency (CF)}, DIG consistently improves final error in all settings, including an \(83\%\) reduction in the \textit{homogeneous} 6-agent load-6000 setting and a \(33\%\) reduction in the \textit{heterogeneous} 6-agent load-222222 setting. Errors mainly come from incomplete coverage and redundant work; DIG improves subtask coverage and prevents premature submission, even in settings with larger loads and more heterogeneous capabilities that create more possible execution paths. On \textbf{Research Job (RJ)}, where errors are further complicated by dependency violations that cause failed processing and merging, DIG can detect these structural failures before they propagate to later stages, with \(44\%\) and \(18\%\) error reductions in the \textit{homogeneous} and \textit{heterogeneous} settings respectively.

\begin{wrapfigure}{r}{0.42\textwidth}
    \centering
    \vspace{-10pt}
    \includegraphics[width=\linewidth]{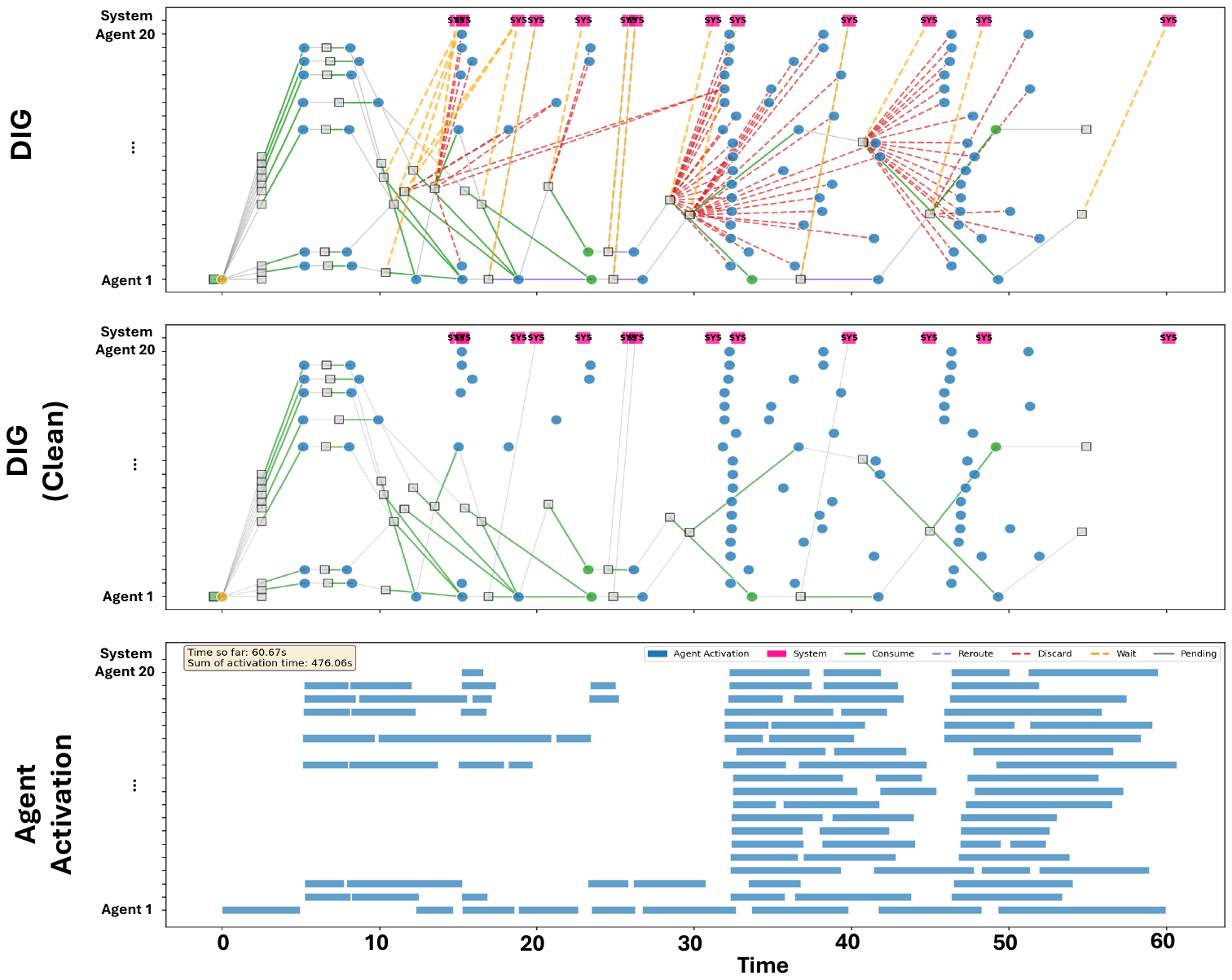}
    \caption{\textbf{MAS + DIG (20 agents):} Large-scale DIG showing structured task propagation, controlled rerouting, and stable convergence despite dense interactions and high concurrency.}
    \label{fig:20agent}
    \vspace{-10pt}
\end{wrapfigure}
These gains come with a trade-off: DIG keeps the system active to repair failures, which may expose additional failures and increase runtime. MAS appears to have fewer detected errors and shorter runtime simply because it leaves failures unresolved and may submit incomplete solutions. The traces in Appendix~\ref{app:exp} show the runtime traces and interaction patterns structurally: without DIG, interactions are more fragmented, with redundant work, rerouting, and unresolved events; with DIG, task propagation is more organized and termination is more stable. Compared with MAS+LLM Judge (Fig.~\ref{fig:n3} and~\ref{fig:llm}), DIG avoids expensive semantic judging and provides more targeted, real-time intervention during execution. The LLM Judge results also show that \textbf{not all interventions are helpful}: untargeted or delayed interventions can increase overhead and even worsen performance, while DIG's structure-driven interventions are more directly tied to the observed failure patterns.

\textbf{Case Study with 20 Agents.} We conduct a 20-agent case study on CF with load \(100{,}000\) to test DIG under high interaction complexity. Without DIG-based healing, MAS fails to produce a result within 120 seconds; with DIG, the same 20-agent system produces a valid answer within 70 seconds using 92 calls. Fig.~\ref{fig:20agent} shows the resulting DIG trace and activation timeline. Even with many agents operating in parallel, DIG detects interaction failures, including 1 missing completion error, 8 repeated-effort incidents, 2 dependency warnings, 1 orphaned event, and 4 early termination errors. These observations highlight that \textit{multi-agent collaboration is non-trivial}: with more agents, the system becomes susceptible to interaction failures, making scalable reliability a central challenge. DIG can identify and heal such failures without assuming a fixed agent count, agent type, or problem domain.

\section{Conclusion}\label{sec:conclusion}

We introduce DIG, a Dynamic Interaction Graph framework for evaluating multi-agent collaboration. DIG models collaboration as a time-evolving causal network of agent activations and interactions, making collaborative dynamics observable and explainable. To the best of our knowledge, DIG is the first framework enabling real-time identification, interpretation, and correction of collaboration-induced error patterns directly from agents' interaction paths. We view DIG as an open framework and welcome extensions to more agents and domains.

\begin{ack}

This work was supported in part by the Office of Naval Research under grant N000142412073 and the National Science Foundation under grants CNS-2533813 and CNS-2312761.
\end{ack}

{\small
\bibliography{references}
\bibliographystyle{plainnat}
}






\clearpage
\appendix

\section{Theoretical Results}\label{app:proof}

\subsection{Proof for Theorem.~\ref{thm:topological_reduction}}
\begin{proof}
Because $\Phi$ depends only on observable interactions, it is constant on each behavioral equivalence class induced by DIG isomorphism. Concretely, for any $\mathcal{T}_1,\mathcal{T}_2\in\mathbb{T}$,
\begin{equation}
\Psi(\mathcal{T}_1)\cong \Psi(\mathcal{T}_2)
\quad\Longrightarrow\quad
\Phi(\mathcal{T}_1)=\Phi(\mathcal{T}_2).
\label{eq:phi_constant_on_classes}
\end{equation}
Define $\mathcal{I}$ on the range of $\Psi$ by selecting any trace $\mathcal{T}\in\mathbb{T}$ such that $\Psi(\mathcal{T})=\{G(t)\}_{t\in\mathbb{Z}}$ and setting
\begin{equation}
\mathcal{I}\bigl(\{G(t)\}_{t\in\mathbb{Z}}\bigr)
\triangleq
\Phi(\mathcal{T}).
\label{eq:def_I}
\end{equation}
This definition is well-defined: if another trace $\mathcal{T}'\in\mathbb{T}$ satisfies $\Psi(\mathcal{T}')=\{G(t)\}_{t\in\mathbb{Z}}$, then $\Psi(\mathcal{T}')\cong\Psi(\mathcal{T})$, so Eq.~\eqref{eq:phi_constant_on_classes} gives $\Phi(\mathcal{T}')=\Phi(\mathcal{T})$. Hence Eq.~\eqref{eq:def_I} does not depend on the chosen representative trace.

Finally, for any trace $\mathcal{T}\in\mathbb{T}$,
\begin{equation}
(\mathcal{I}\circ \Psi)(\mathcal{T})
=
\mathcal{I}\bigl(\Psi(\mathcal{T})\bigr)
=
\Phi(\mathcal{T}),
\end{equation}
where the second equality follows from Eq.~\eqref{eq:def_I}. Therefore, $\Phi=\mathcal{I}\circ\Psi$.
\end{proof}




\clearpage
\section{General Agent}\label{app:agent}

\begin{tcolorbox}[
  breakable,
  title={Agent Instructions Template},
  colback=white,
  colframe=black,
  fonttitle=\bfseries,
  boxrule=0.8pt,
  left=1em,
  right=1em,
  top=1em,
  bottom=1em
]
\ttfamily\footnotesize
You are \{agent\_name\}, an autonomous agent in a cooperative multiagent problem-solving system.

GOAL: Work with your collaborators to solve the problem in the SHORTEST TIME (minimum stages) possible. Coordinate efficiently and avoid redundant work.

PROBLEM CONTEXT

\{problem\_spec\}

\{stage\_progress\}

AVAILABLE AGENTS: \{agent\_list\}

INPUT EVENT ACTIONS

For each pending event, specify ONE action:

``consume'' - Process this event. MUST produce output event(s) OR use a tool. If you consume, you MUST generate something (event or tool call). If you're not ready to process it, use ``wait'' instead. If you don't want it, use ``discard'' instead.
``reroute'' - Forward to other agents. MUST specify reroute\_to recipients.
``discard'' - Drop permanently (for duplicates/irrelevant events).
``wait'' - Keep in buffer for later (need more data first).

HOW TO PICK ACTION FOR EACH EVENT

For PROBLEM events (without raw data): CONSUME: you MUST use a tool (split\_problem\_at\_high\_level, use\_skill, or get\_raw\_data\_of\_problem). REROUTE: forward to another agent if you're busy. WAIT: keep for later if you have higher priority work.

For SKILL RESULT events (info.skill is set): CONSUME: read info.skill\_result, then merge with other partials or submit. WAIT: if you are expecting more partial results before merging.

For RAW DATA events: CONSUME: compute solution and output a solution event; share your solution to others to speed up the process. WAIT: if you need more data first.

For SOLUTION events: CONSUME: aggregate with other solutions you have. DISCARD: if it's a duplicate of information you already have. WAIT: if you're expecting more solutions to aggregate.

AGGREGATION: When consuming multiple solution events, reduce redundancy and output ONE combined solution.

SUBMIT: Set is\_final\_answer=True when your solution is complete or good enough. To submit after a merge: include yourself in merge recipients, then in the NEXT activation call use\_skill(submit, result\_event\_id=<merged\_event\_id>) and set is\_final\_answer=True. NEVER call merge and submit in the same activation --- always two separate activations.

HANDLING MULTIPLE INPUTS: Prioritize aggregation if you have multiple solutions. Stay focused on one task type per activation.

EVENT TYPES

``problem'' - A task to solve. Use use\_skill(decompose/process) to handle it server-side.
``skill result'' - Compact result from use\_skill(process/merge). Check info.skill\_result. Merge partials with use\_skill(merge) or use use\_skill(submit) when complete.
``solution'' - Final answer. ONLY create via use\_skill(submit) or is\_final\_answer=True.
``process\_error'' - Your process call was rejected because the problem exceeds your capacity. info.error\_type = `capacity\_exceeded', info.problem\_id = the problem you tried to process, info.problem\_size = actual number of elements, info.agent\_capacity = your maximum capacity. You MUST handle this event. Two options: (a) Decompose: use\_skill(decompose, problem\_id, assignments=\{AgentX: size1, AgentY: size2, ...\}) so each chunk size <= the recipient's capacity. (b) Reroute: use\_skill(decompose, problem\_id, assignments=\{CapableAgent: full\_size\}) to delegate the whole problem to one agent with sufficient capacity.
``merge\_error'' - Your merge call was rejected because one or more tasks have not been processed yet (no skill\_result available). info.error\_type = `unprocessed\_tasks', info.attempted\_event\_ids = the event IDs you passed to merge, info.unprocessed = list of (event\_id, load) pairs for unprocessed tasks, info.message = human-readable description of what went wrong. You MUST handle this event. First process the unprocessed tasks: use\_skill(process, problem\_id) for each unprocessed task, then retry the merge once all tasks have skill\_results.

\{event\_examples\}

TOOLS

If you consume a PROBLEM event without raw data, you MUST use a tool (split\_problem\_at\_high\_level, use\_skill, or get\_raw\_data\_of\_problem). Prefer use\_skill(decompose) over split\_problem\_at\_high\_level for automated splitting. Prefer use\_skill(process) over get\_raw\_data\_of\_problem to avoid raw data in the prompt. Do NOT also create out\_events for the same consumed PROBLEM. You MUST still create out\_events for any other consumed inputs not handled by a tool (e.g., aggregated solution events). NEVER embed raw result data (counts, arrays, computed values) in out\_events payloads. To share a result with another agent, include them in the recipients of use\_skill(process) or use\_skill(merge) --- the skill result event will be delivered to them directly. You can also reroute a received skill result event to additional agents. Other agents merge results using use\_skill(merge, event\_ids=[...]) with the event IDs.

\{tools\_description\}

\{capability\_explain\}

OUTPUT FORMAT

input\_actions: REQUIRED. \{event\_id, action, reroute\_to (if rerouting)\}
is\_final\_answer: Set to True when you want to SUBMIT a final solution. If you have a reasonable answer, submit it!
tool\_requests OR out\_events (mutually exclusive)
reasoning: \{observation, thought, action\}
FINAL SUBMISSION payload must include: \{"type": "solution", "problem\_id": "P", "solution": \{...\}\}
NO EMPTY ACTIVATIONS: You MUST use tools OR generate output events (or both). You cannot do neither.

REASONING GUIDANCE

Be compact. Include only important details needed to justify actions and help teammates.
\end{tcolorbox}

\newpage
\section{LLM Judge}
\label{app:llm-judge-prompt}

\begin{tcolorbox}[
  title={Error Taxonomy for Cooperative Multi-Agent Systems},
  colback=white,
  colframe=black,
  fonttitle=\bfseries,
  boxrule=0.8pt,
  left=1em,
  right=1em,
  top=1em,
  bottom=1em
]

This document defines a structural failure taxonomy for cooperative multi-agent systems in terms of
observable interaction patterns in the Dynamic Interaction Graph (DIG), independent of agent internals
or task semantics.

\medskip

\textbf{Reachability and Termination (Failures)}

All reachable work in the DIG should persist until some activation consumes it. The system should emit
a single \textsc{Submit} event, and only after all reachable work has been consumed.

\begin{itemize}
\item \textbf{Early termination (ET)}: a \textsc{Submit} event is generated while some reachable work remains unconsumed.

\item \textbf{Missing termination (MC)}: all reachable work has been consumed, but no \textsc{Submit} event is generated within a reasonable time window.

\item \textbf{Orphaned event (OE)}: an event is generated but becomes unreachable before any activation can consume it, either because it is discarded or because it is assigned invalid recipients.

\item \textbf{Deadlock (DL)}: reachable work remains, but all activated agents wait indefinitely, making progress impossible.
\end{itemize}

\medskip

\textbf{Progress (Warnings)}

Generated events should be consumed downstream within a reasonable time; repeated deferral, rerouting,
or redundant handling indicates risk.

\begin{itemize}
\item \textbf{Excessive rerouting (ER)}: an event is repeatedly \textsc{Reroute}d across activations without any activation \textsc{Consume}ing it to generate downstream events.

\item \textbf{Cross-lineage aggregation (CLA)}: an activation \textsc{Consume}s events from different lineages.

\item \textbf{Repeated subproblem solving (RSP)}: multiple problem-reducing activations generate outputs whose coverage overlaps on the same upstream event.
\end{itemize}

\medskip

\textbf{Intended Use}
\begin{itemize}
\item Monitoring emergent cooperation
\item Early failure detection
\item Structure-driven system healing
\item Explainable debugging of multi-agent executions
\end{itemize}

\end{tcolorbox}

\begin{tcolorbox}[
  title={LLM Judge System Prompt},
  colback=white,
  colframe=black,
  fonttitle=\bfseries,
  boxrule=0.8pt,
  left=1em,
  right=1em,
  top=1em,
  bottom=1em
]
\small
You are an expert system monitor for cooperative multi-agent systems.
Your role is to analyze the Dynamic Interaction Graph (DIG) and detect
structural errors based solely on observable interaction patterns.

You do not have access to agent internals, reasoning traces, task semantics,
or predefined workflows. All judgments must be derived from DIG topology,
event lineage, and interaction primitives (CONSUME, DELAY, REROUTE, DISCARD,
SUBMIT).
\medskip

\textbf{Error Taxonomy \{file\_input\}}
\medskip

\textbf{DIG Log \{file\_input\}}
\medskip

\textbf{Your Task}
\begin{itemize}
\item Identify errors using the taxonomy above
\item Assess severity (critical/high/medium/low)
\item Recommend conservative interventions only when necessary
\item Provide clear reasoning for your decisions
\end{itemize}

\medskip
\textbf{Intervention Methods}
\begin{itemize}
\item \textbf{inject\_system\_message}: Add guidance to existing event and optionally reroute to different agents
\item \textbf{create\_system\_event}: Create new intervention event with message to specific agents
\end{itemize}

\medskip
\textbf{Intervention Methods}

\begin{itemize}
\item \textbf{inject\_info}: Inject new information into the blocked event before delivery.
\item \textbf{inject\_and\_reroute}: Inject guidance into the blocked event and reroute it to a specified recipient set.
\item \textbf{create\_system\_event}: Create a new system-level event and deliver it to a specified recipient set.
\end{itemize}

\medskip
\textbf{Principles}
\begin{itemize}
\item Intervene when error is clear and intervention helps
\item Prioritize FAILURES over RISKS
\item Provide actionable, specific recommendations
\item Respect agent autonomy where appropriate
\end{itemize}

\end{tcolorbox}

\newpage
\section{Tool Semantics}
\label{app:tool_semantics}

Agents may invoke tools only during \textsc{Respond} activations. We use three tools to modify task structures and states:
\[
    \textsc{Decompose}(g,\{i:\ell_i\}_{i\in I}),
    \qquad
    \textsc{Process}(g),
    \qquad
    \textsc{Merge}(\{g_k\}_{k=1}^{K}).
\]

\textbf{\textsc{Decompose}.}
\textsc{Decompose} partitions task \(g\) into subtasks assigned to agents \(i\), where \(\ell_i\) is the load assigned to agent \(i\). If the total assigned load exceeds the load of \(g\), the decomposition fails. If the total assigned load is smaller, the unassigned remainder is generated as a new task and routed back to the source agent.

\textbf{\textsc{Process}.}
\textsc{Process} attempts to change a task component from \textsc{Unprocessed} to \textsc{Processed}. The call succeeds only if the acting agent has sufficient capability for the task load and all dependency constraints of the task component are satisfied.

\textbf{\textsc{Merge}.}
\textsc{Merge} combines multiple task components into a larger result. The call succeeds only if all input components are processed and structurally compatible. Otherwise, the merge fails and a failure message is returned to the source agent.

\section{Significance Tests}
\label{app:significance-tests}

For each setting, we treat experiment runs as independent samples with and without DIG.
We use one-sided Mann--Whitney U tests to evaluate each metric separately, under the hypotheses that DIG leads to lower task error, higher detected errors, and longer runtime.
We further use MANOVA to test the joint effect on all three metrics. Table~\ref{tab:significance-tests} reports our results.

Overall, DIG consistently yields significantly higher detected errors in most settings ($p < 0.05$), indicating improved failure visibility.
Improvements in task error are observed in some settings but not uniformly.
This suggests that while DIG reliably exposes structural failures, local interventions are not always sufficient to fully resolve them, highlighting the need for more advanced healing strategies.
We also observe that MANOVA shows significant effects in more complex settings, indicating that DIG has more potential under higher interaction complexity.

\begin{table}[H]
\centering
\caption{Significance test results comparing runs with and without DIG. Values report one-sided Mann--Whitney U test $p$-values for individual metrics and MANOVA $p$-values for the joint effect across all three metrics. Asterisks indicate significance at $p < 0.05$.}
\label{tab:significance-tests}
\begin{tabular}{llcccc}
\toprule
\textbf{Domain} & \textbf{Setting} & \textbf{Detected Errors} & \textbf{Task Error} & \textbf{Runtime} & \textbf{MANOVA} \\
 & & \textbf{($p$)} & \textbf{($p$)} & \textbf{($p$)} & \textbf{($p$)} \\
\midrule
CF 4-agent & homo 4000   & 0.531 & 0.124 & 0.988 & 0.069 \\
CF 4-agent & homo 8000   & 0.000$^{*}$ & 0.037$^{*}$ & 0.001$^{*}$ & 0.000$^{*}$ \\
CF 4-agent & hetero 1111 & 0.178 & 0.864 & 0.715 & 0.712 \\
CF 4-agent & hetero 2222 & 0.000$^{*}$ & 0.038$^{*}$ & 0.007$^{*}$ & 0.000$^{*}$ \\
CF 6-agent & homo 6000   & 0.500 & 0.271 & 0.121 & 0.488 \\
CF 6-agent & homo 12000  & 0.000$^{*}$ & 0.004$^{*}$ & 0.032$^{*}$ & 0.000$^{*}$ \\
CF 6-agent & hetero 111111 & 0.020$^{*}$ & 0.867 & 0.052 & 0.064 \\
CF 6-agent & hetero 222222 & 0.000$^{*}$ & 0.061 & 0.013$^{*}$ & 0.000$^{*}$ \\
Research & homo 1000 & 0.021$^{*}$ & 0.066 & 0.009$^{*}$ & 0.041$^{*}$ \\
Research & hetero 1111 & 0.000$^{*}$ & 0.096 & 0.013$^{*}$ & 0.000$^{*}$ \\
\bottomrule
\end{tabular}
\end{table}

\newpage
\section{Detailed Comparison with Description Logics and Process Algebras}
\label{app:formalism_comparison}

\textbf{Description Logics (DLs).}
DLs model entities as persistent nodes and relationships as directed edges within a well-specified schema, enabling strong decidability guarantees but at the cost of expressiveness for interaction modeling~\citep{baader2008description}. Since entities are often collapsed into a single persistent node, DLs do not naturally capture interaction histories or evolving interaction structure. While some DL-based frameworks, such as OWL~\citep{antoniou2009web} or temporal DLs~\citep{artale2000survey}, allow modeling evolving properties or actions, the object of analysis is typically state change rather than the evolving interactions themselves.

In DIG, agent activations appear as distinct, time-stamped nodes, and interactions are mediated through event nodes with explicit causal dependencies. Relationships in DIG are dynamically induced by agents through events during execution. By preserving activations separately, DIG tracks how events are transformed, divided, delayed, rerouted, or consumed across agents, ensuring that failure patterns such as repeated rerouting, orphaned events, or cross-lineage aggregation remain distinguishable and can be detected and, ideally, corrected in real time as agents interact.

\textbf{Process Algebras (PAs).}
PAs model concurrent systems and verify their properties~\citep{de2014gentle}, and they have been applied to performance evaluation of stochastic systems~\cite{hermanns2002process}. Many classical approaches rely on interleaving semantics, representing concurrency as nondeterministic sequential patterns, and typically assume predefined communication primitives~\cite{fokkink2013introduction}. While PAs excel at verifying protocols, they often abstract internal agent-to-agent transitions as silent actions, treating structurally different interaction paths as equivalent if they produce the same observable outcomes.

To analyze LLM cooperative problem solving, understanding how a task changes is critical. Instead of verifying predefined protocols, DIG observes and constructs the true interaction structure as it dynamically unfolds. While abstracting away internal node decisions, DIG explicitly preserves inter-node interactions, treating the task as an evolving entity whose lifecycle is captured by the interaction structure. By tracking how interactions propagate and by using edge labels to distinguish productive from non-productive interactions, DIG enables reasoning about task progress and structural inefficiencies, such as effective decomposition versus redundant work or excessive rerouting. Notably, graph rewrite operators allow dynamic modifications to how previously generated events are handled, such as rerouting.

Compared to PAs or DLs, DIG is more expressive for modeling the sequential, time-evolving nature of stochastic and emergent interactions in unconstrained LLM agents at runtime. Because relationships in DIG are dynamically generated rather than defined within a fixed schema or protocol, DIG prioritizes reasoning over evolving interaction structure rather than formal decidability or protocol guarantees. We view DLs and PAs as complementary to our setting, as they may help define and reason about additional structural invariants on top of DIG.

    


\newpage
\section{Example traces}\label{app:exp}

\begin{figure}[h]
    \centering
    \includegraphics[width=0.5\linewidth]{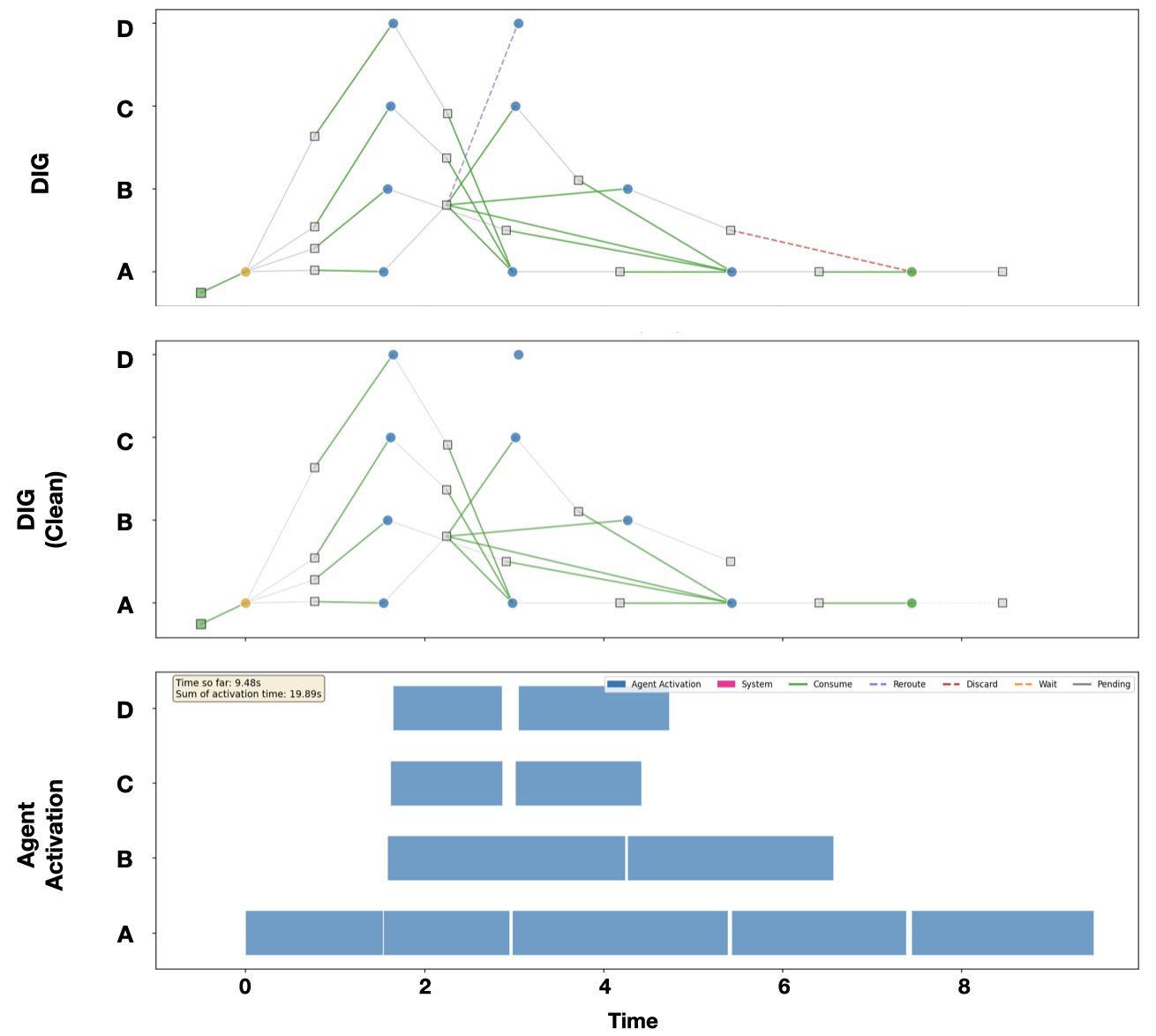}
    \caption{\textbf{Count Frequency example (Homogeneous\_4000, 4 agents).} The problem is within the group’s capability. Agent 1 decomposes the task for parallel execution, but due to asynchrony and partial observability, agents may ignore subtasks, reroute unnecessarily, or duplicate work. These stochastic behaviors make cooperation brittle, which DIG captures and helps diagnose.}
\end{figure}

\begin{figure}[h]
    \centering
    \begin{minipage}[t]{0.49\linewidth}
        \centering
            \includegraphics[width=\linewidth]{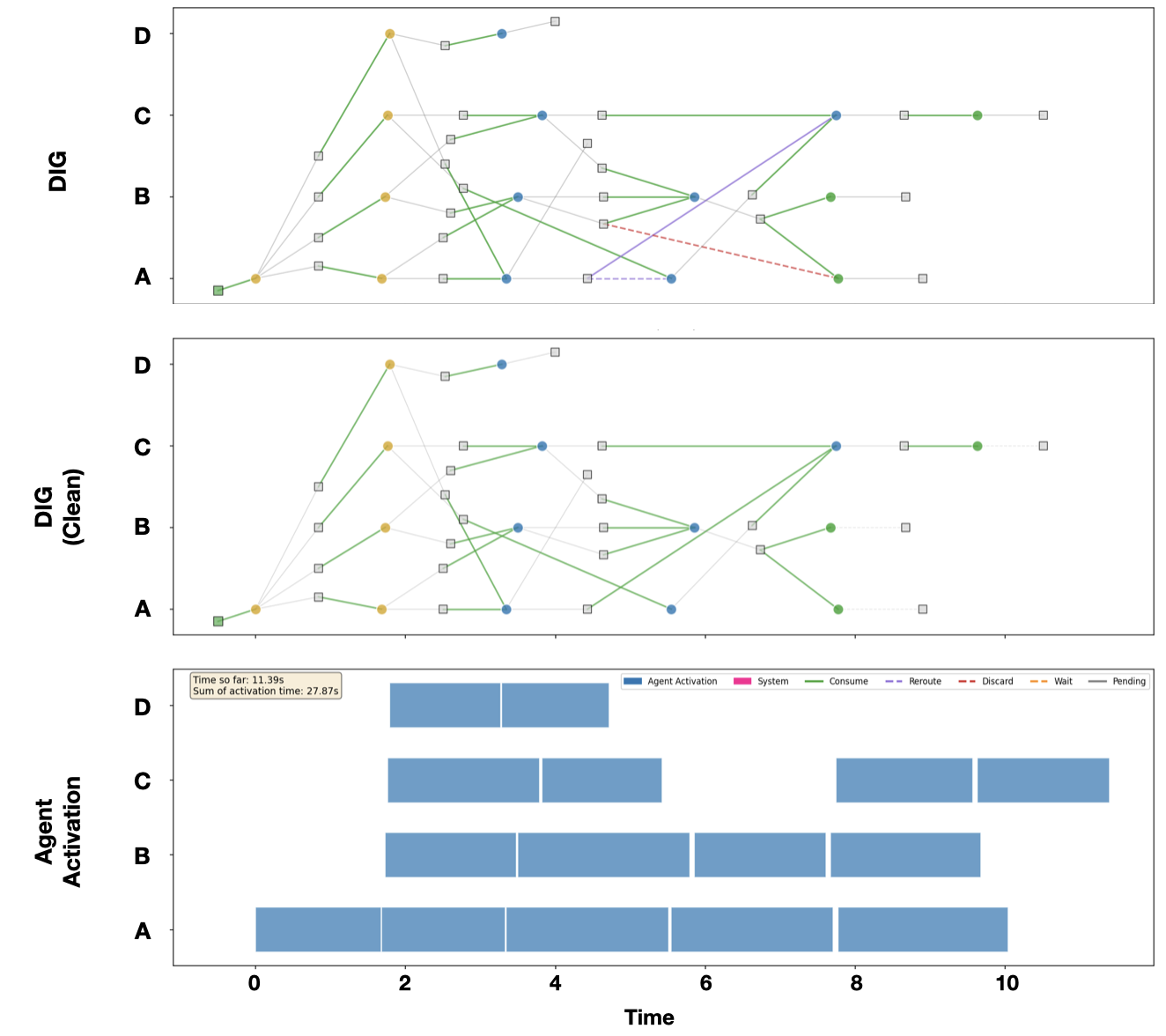}
        \caption{\textbf{Count Frequency example (Homogeneous\_8000, 4 agents, w/o DIG).} The problem exceeds the group’s capability. Agents continue to decompose tasks, resulting in a more complex structure of problem-solving and interaction.}
    \end{minipage}
    \hfill
    \begin{minipage}[t]{0.49\linewidth}
        \centering
        \includegraphics[width=\linewidth]{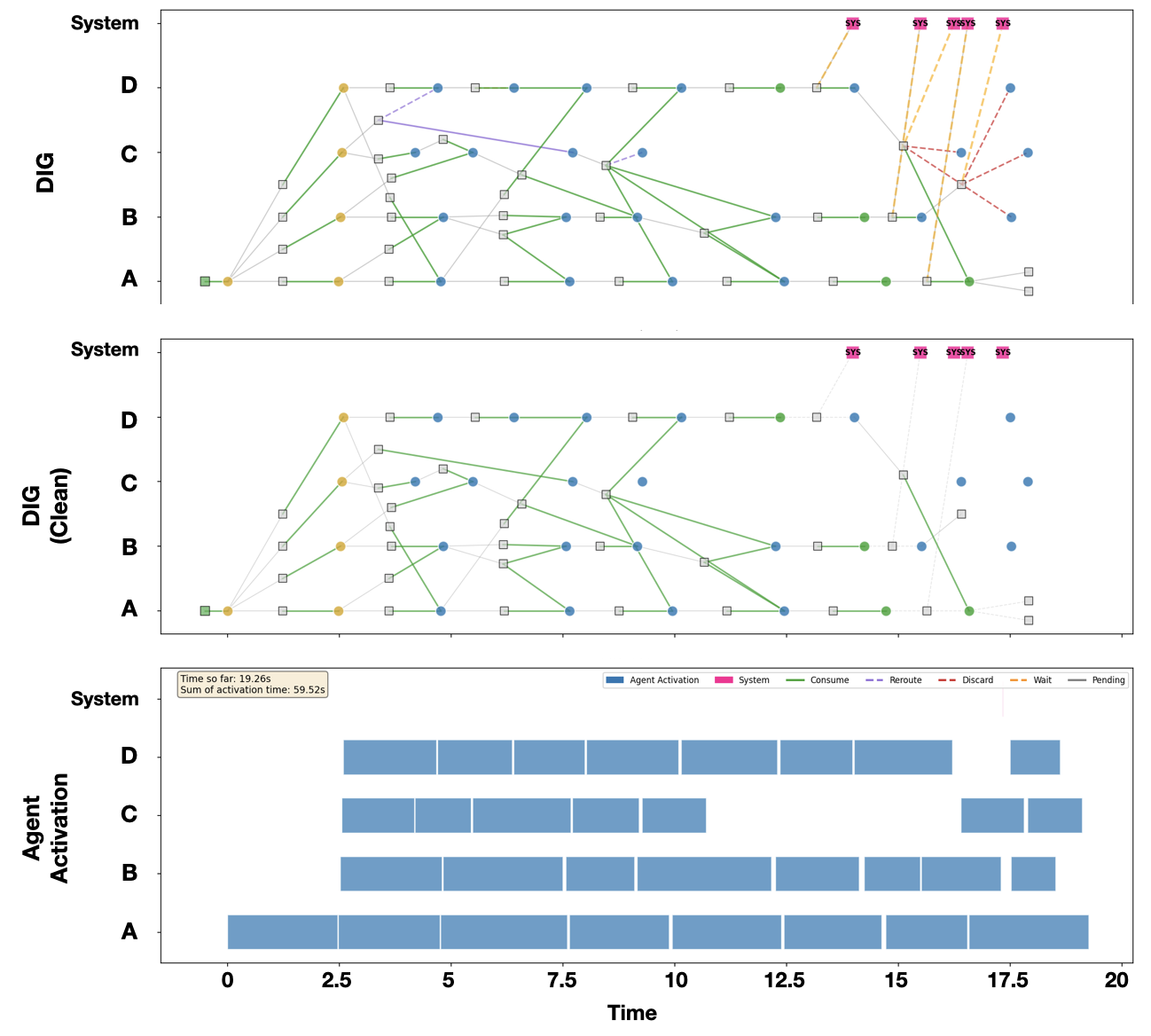}
        \caption{\textbf{Count Frequency example (Homogeneous\_8000, 4 agents, w/ DIG).} DIG mitigates early termination and redundant work, improving the cooperation structure and overall performance.}
    \end{minipage}
\end{figure}

\begin{figure}[h]
    \centering
    \begin{minipage}[t]{0.49\linewidth}
        \centering
        \includegraphics[width=\linewidth]{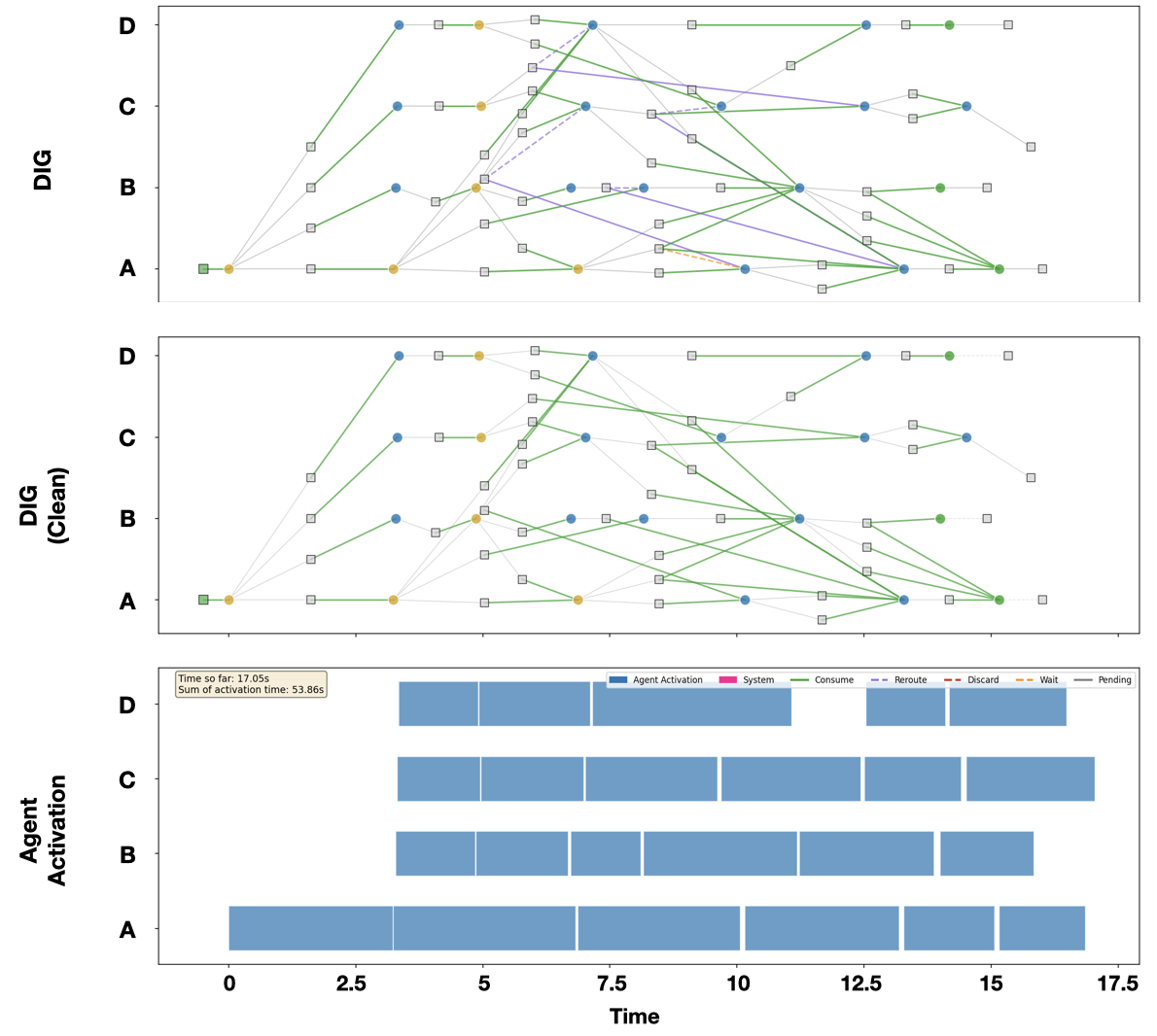}
        \captionof{figure}{\textbf{Count Frequency example (Heterogeneous\_2222, 4 agents, w/o DIG).} The problem exceeds group capability. Agents decompose tasks based on downstream capacity, leading to uneven interaction patterns.}
    \end{minipage}
    \hfill
    \begin{minipage}[t]{0.49\linewidth}
        \centering
        \includegraphics[width=\linewidth]{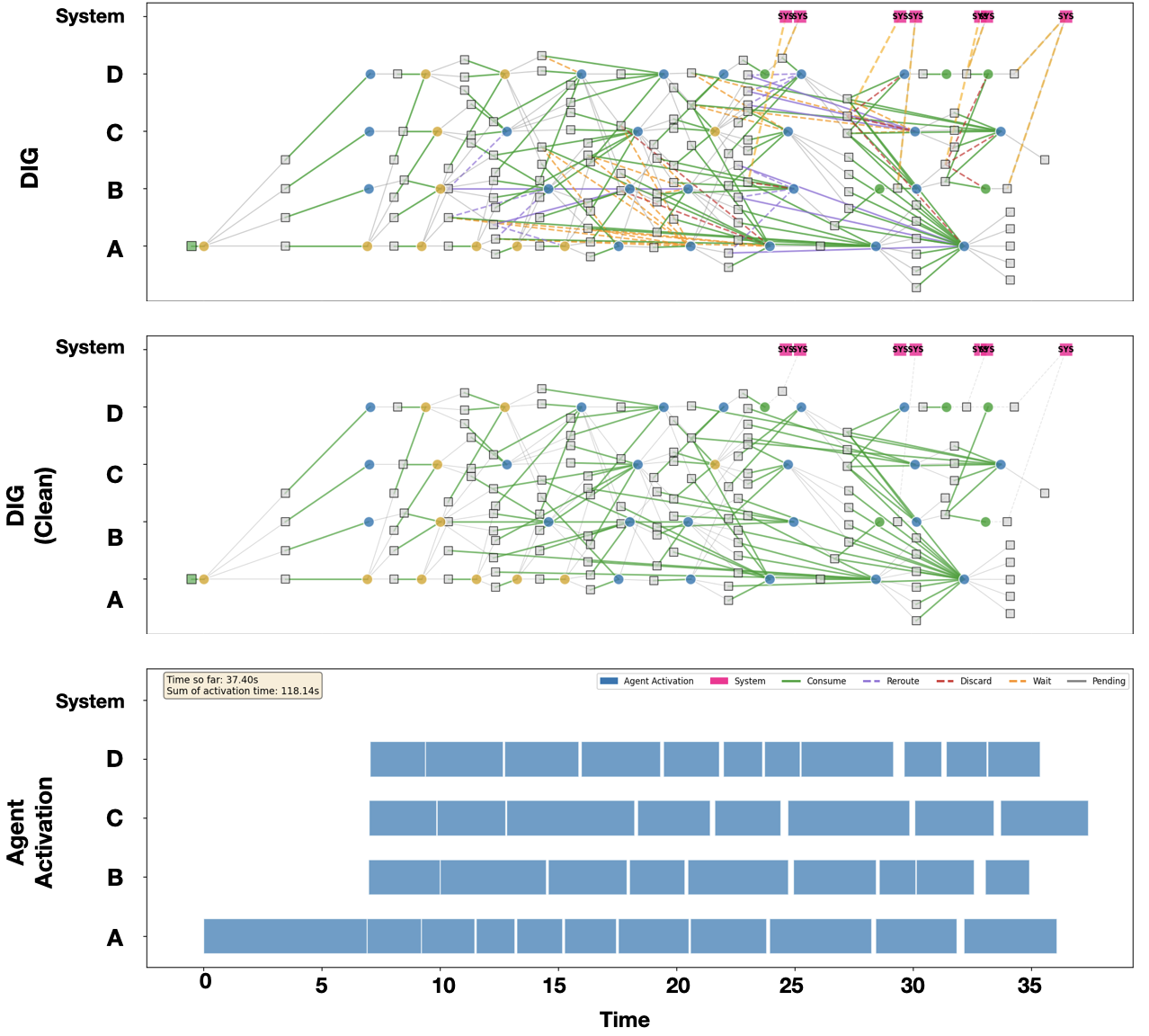}
        \captionof{figure}{\textbf{Count Frequency example (Heterogeneous\_2222, 4 agents, w/ DIG).} DIG monitors interactions and intervenes in real time, improving coordination despite increased complexity.}
    \end{minipage}
\end{figure}

\begin{figure}[t]
    \centering
    \begin{minipage}[t]{0.49\linewidth}
        \centering
        \includegraphics[width=\linewidth]{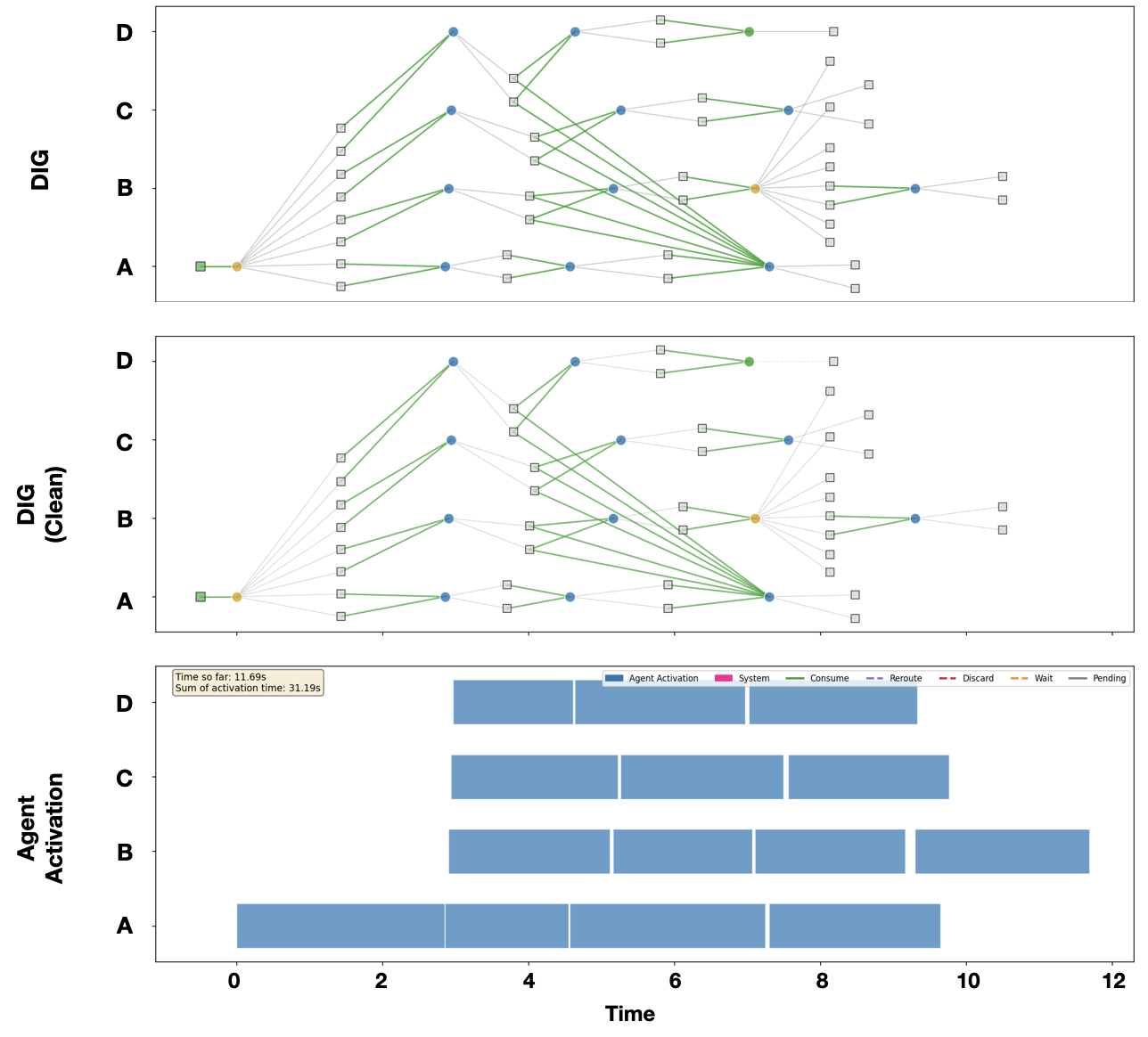}
        \caption{\textbf{Research domain (Heterogeneous\_1111, 4 agents, w/o DIG).} Agents initially operate on the same stage and progressively move forward. However, due to the more complex task structure, errors are more likely: agents may merge incompatible results, violate dependencies, or terminate early under partial observability.}
    \end{minipage}
    \hfill
    \begin{minipage}[t]{0.49\linewidth}
        \centering
        \includegraphics[width=\linewidth]{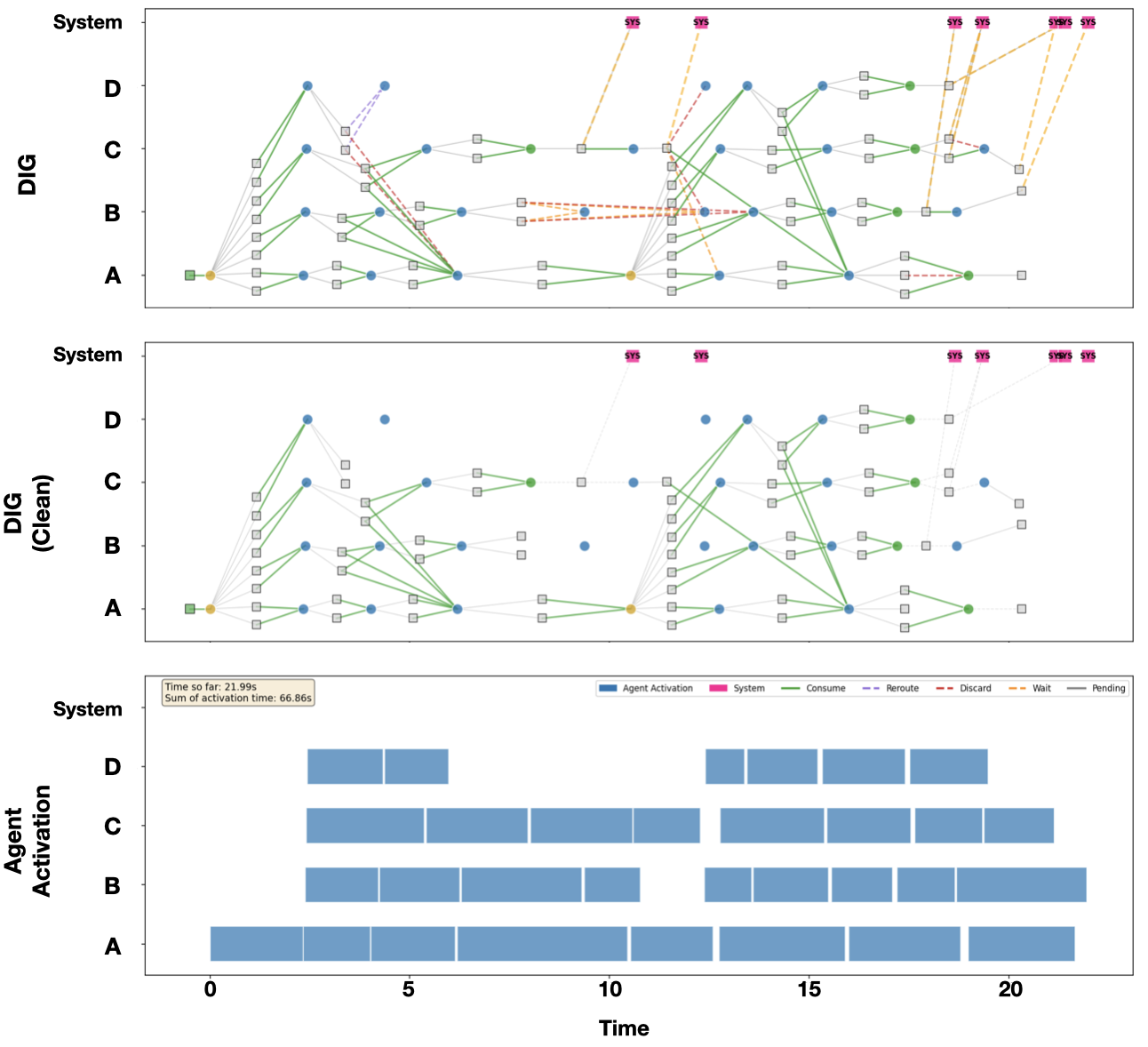}
        \caption{\textbf{Research domain (Heterogeneous\_1111, 4 agents, w/ DIG).} DIG monitors interaction structure and detects violations of graph-level invariants, improving coordination even in complex task settings.}
    \end{minipage}
\end{figure}





\end{document}